\documentclass[twoside]{article}

\usepackage[preprint]{aistats2026}

\usepackage[american]{babel}

\usepackage[round]{natbib} 
\usepackage{mathtools} 
\usepackage{booktabs} 
\usepackage{tikz} 

\usepackage{graphicx} 
\usepackage[utf8]{inputenc} 
\usepackage[T1]{fontenc}    
\PassOptionsToPackage{colorlinks=true,linkcolor=red,citecolor=magenta}{hyperref}
\usepackage{booktabs}       
\usepackage{amsfonts}       
\usepackage{amsthm}         
\usepackage{nicefrac}       
\usepackage{microtype}      
\usepackage{xcolor}         
\usepackage{amsmath}        
\usepackage{amssymb}        
\usepackage{mathtools}      
\usepackage{longtable} 
\usepackage{float}     

\usepackage{algorithm}      
\usepackage[noend]{algpseudocode} 

\usepackage{etoc}          
\usepackage{tikz}  
\usepackage{pgfplots} 
\pgfplotsset{compat=1.18} 
\usetikzlibrary{calc} 
\usepackage{subcaption} 
\usepackage{hyperref}   

\newtheorem{theorem}{Theorem}
\newtheorem{lemma}{Lemma}
\newtheorem{remark}{Remark}
\newtheorem{definition}{Definition}
\newtheorem{proposition}{Proposition}

\newtheorem{assumption}{Assumption}

\DeclarePairedDelimiter{\brk}{[}{]}
\DeclarePairedDelimiter{\crl}{\{}{\}}

\DeclareMathOperator*{\argmax}{arg\,max}

\newcommand{\on}{\mathrm{on}} 
\newcommand{\off}{\mathrm{off}} 
\newcommand{\boff}{b_h^{\off}} 
\newcommand{\bont}{b_h^{t,\on}} 
\newcommand{\boffp}[1]{b_{#1}^{\off}} 
\newcommand{\ovpihs}{\overline{\pi}_h(s)} 
\newcommand{\BPS}{\mathrm{BPS}_{\Delta}} 
\newcommand{\PPS}{\mathrm{PPS}_{\Delta}} 
\newcommand{\PS}{\mathrm{PS}_{\Delta}} 
\newcommand{\lowq}{\widetilde{W}} 
\newcommand{\highq}{\widetilde{U}} 
\newcommand{\lowv}{W} 
\newcommand{\highv}{U} 
\newcommand{\ohighv}{\overline{V}} 
\newcommand{\olowv}{\underline{V}} 
\newcommand{\ohighq}{\overline{Q}} 
\newcommand{\olowq}{\underline{Q}} 

\newboolean{voirnote}
\setboolean{voirnote}{true} 

\newcounter{notecounter}
\usepackage{framed}
\usepackage{ifthen}

\ifthenelse{\boolean{voirnote}}
{
\newcommand{\note}[2]{\refstepcounter{notecounter} \begin{leftbar}  \textcolor{red}{\textbf{#1}} \textcolor{violet}{(\textbf{Note} \thenotecounter)}: \begin{sf}{\color{blue} #2} \end{sf}  \end{leftbar}}
}{
\newcommand{\note}[2]{}
}

\newcommand{\sandwich}{bounding interval }

\begin{document}

\twocolumn[

\aistatstitle{Learning Upper–Lower Value Envelopes to Shape Online RL: A Principled Approach}

\aistatsauthor{ Sebastian Reboul \And Hélène Halconruy }

\aistatsaddress{
  SAMOVAR, \\
  Télécom SudParis, \\
  Institut Polytechnique de Paris, \\
  91120 Palaiseau \\[1.5ex]
  Wiremind Cargo \\
  16 Bd Poissonnière \\
  75009 Paris
  \And
  SAMOVAR, \\
  Télécom SudParis, \\
  Institut Polytechnique de Paris, \\
  91120 Palaiseau \\[1.5ex]
  Modal'X, \\
  Université Paris-Nanterre, \\
  92000 Nanterre
} ]

\begin{abstract}
We investigate the fundamental problem of leveraging offline data to accelerate online reinforcement learning - a direction with strong potential but limited theoretical grounding. Our study centers on how to \emph{learn} and \emph{apply} value envelopes within this context. To this end, we introduce a principled two-stage framework: the first stage uses offline data to derive upper and lower bounds on value functions, while the second incorporates these learned bounds into online algorithms. Our method extends prior work by decoupling the upper and lower bounds, enabling more flexible and tighter approximations. In contrast to approaches that rely on fixed shaping functions, our envelopes are data-driven and explicitly modeled as random variables, with a filtration argument ensuring independence across phases. The analysis establishes high-probability regret bounds determined by two interpretable quantities, thereby providing a formal bridge between offline pre-training and online fine-tuning. Empirical results on tabular MDPs demonstrate substantial regret reductions compared with both UCBVI and prior methods while remaining competitive with related approaches.





\end{abstract}

\section{Introduction}

%

%

\subsection{Motivation}

\noindent
Reinforcement Learning (RL) provides a general framework where agents interact with an environment to accomplish specific goals. Most existing theoretical guarantees for RL algorithms are framed in the \emph{worst case}, meaning they apply regardless of the structure of the underlying problem \citep{azar_minimax_2017,jin2018q,russo2019worst,zhang2024settling}. While this universality is appealing, such bounds are often overly pessimistic: they ignore problem-specific features that could make learning dramatically easier.

One way to go beyond worst-case analysis is through \emph{gap-dependent regret bounds} \citep{simchowitz2019non,zanette2019tighter}, which exploit the sub-optimality gaps between actions or policies. Another complementary approach is to assume that the learner has access to some form of \emph{prior knowledge} about the environment. A natural and powerful instance of such prior knowledge is given by \emph{shaping functions}, approximations of the optimal value or $Q$-functions.

Building on this idea, \cite{gupta_unpacking_2022} introduced a framework showing how approximate shaping functions can be leveraged to rescale exploration bonuses and clip value estimates. By focusing exploration on the “effective” subset of states or state–action pairs that matter, shaping functions effectively prune away regions that are likely suboptimal, yielding significantly tighter regret bounds. However, as \cite{gupta_unpacking_2022} point out, their analysis assumes that a sufficiently accurate shaping function is provided \emph{a priori}. This leads to a natural and pressing question:



\begin{center}
    \emph{Can we \textbf{learn} a shaping function in a principled way?}\vspace{-0.2cm}
\end{center}
\subsection{Contributions}
In many practical scenarios, we may have access to a batch of historical data from a related task or a previous policy. This offline data, while perhaps insufficient for learning a near-optimal policy directly, could contain valuable information for guiding a subsequent, more expensive online interaction phase.

In this work, we propose a two-stage framework that uses offline data to learn a shaping function. Our key contributions are:
\begin{enumerate}
    \item We use a model-based offline RL algorithm to process a batch dataset of $K$ trajectories and produce bounds on the optimal value function. Importantly, these value estimates are \textbf{optimistic} and \textbf{pessimistic}, contrary to the standard pessimistic approach in offline reinforcement learning. We call these bounds \emph{value envelopes}.
    \item We use the framework from \cite{gupta_unpacking_2022} to integrate these estimates into an online algorithm using them as shaping functions. Contrary to the original work, these shaping functions are thus random variables.
    \item We derive regret bounds that connect the offline dataset to online sample complexity using a custom empirical Bernstein that remains valid when bounds are random. Our bounds introduce two key quantities, related to the quality of the learned envelopes and the MDP.
    \item We validate our results by examining the effect of these two quantities as well as the offline dataset size on regret through experiments on parametrized layered tabular MDPs, demonstrating significant improvements over Bernstein UCBVI \cite{azar_minimax_2017} and prior approaches as well as a version of UCBVI with augmented initial counts (see Section~\ref{sec:experiments}).
\end{enumerate}

Our work provides bounds that explicitly connect offline sample size to online sample complexity, demonstrating how a limited batch of data can be provably leveraged to accelerate subsequent exploration and learning without explicitly leveraging the offline data in the online phase.

\begin{remark}[Privacy Compatibility]
    An important aspect of our method is that the transition data and estimates from the offline phase are not reused in the online phase. This allows for situations in which the actual data found in the state transitions pertain to sensitive or confidential data. Only passing on value envelopes maintains \textbf{privacy} and fulfills our main working hypothesis (see Assumption~\ref{assumption:sandwich}).
\end{remark}

\subsection{Organization of the paper}
The paper is organized as follows. Section \ref{sec:setting} introduces the RL framework and algorithms. Section \ref{sec:method} presents our envelope-shaping approach, while Section \ref{sec:bounds} derives regret bounds, focusing on the $Q$-shaping variant (with $V$-shaping detailed in Section \ref{sec:V_shaping} and the Appendix). Section \ref{sec:experiments} reports comparisons with Bernstein UCBVI, Hoeffding UCBVI and \cite{gupta_unpacking_2022}, and Section \ref{sec:discussion} concludes.

\subsection{Related work}



\paragraph{Gap-dependent regret bounds}
To address this limitation, researchers have investigated \emph{gap-dependent regret bounds} \citep{simchowitz2019non,zanette2019tighter}, which incorporate the sub-optimality gap of actions or policies. By taking this separation into account, such bounds often become much tighter when suboptimal choices are clearly worse, capturing the intuition that learning is easier when the optimal policy is well-distinguished from the rest \citep{dann2021beyond,lykouris2021corruption}.

\paragraph{Prior knowledge}
Beyond gap-based analyses, another approach to moving past worst-case guarantees is to assume that the learning algorithm has access to certain forms of \emph{prior knowledge} about the problem instance. For instance, the algorithm may be given predictions of the optimal $Q$-function that are accurate on a large but unknown subset of state–action pairs and satisfy a \emph{distillation} condition \citep{golowich2022can}, ensuring sufficient agreement with the true $Q^\star$ to let the learner effectively “distill” useful knowledge. Such prior knowledge can also take the form of shaping functions derived from the optimal value or $Q$-value functions \citep{gupta_unpacking_2022}. In a similar vein, \emph{heuristics} \citep{kolobov2012planning} (i.e., approximate value functions) constructed from domain knowledge or offline data can effectively shorten the problem horizon, enabling faster learning in online RL \citep{cheng2021heuristic}.

\paragraph{Accelerating online RL with offline knowledge}
These perspectives connect naturally to a growing body of work on leveraging offline data to accelerate online reinforcement learning. Training RL agents from scratch often requires large amounts of interaction data \citep{zhai2022computational,ye2020towards,kakade2002approximately}, whereas offline datasets can provide strong initializations that allow faster and more efficient online fine-tuning \citep{nakamoto2023cal, huang_augmenting_2025, pattanaik_2025}. This can be achieved by incorporating offline trajectories into replay buffers \citep{song2022hybrid,schaal1996learning, lee_2022}, adding behavioral cloning losses \citep{kang2018policy,zhu2019dexterous}, or extracting skill representations for downstream learning \citep{ajay2020opal,gupta2019relay}.

\paragraph{Reward shaping}
Another line of work is \emph{reward shaping}, which augments the reward function with an additional term to accelerate learning while preserving the optimal policy \citep{ng1999policy}. \citet{brys2015reinforcement} use demonstrations - trajectories of state–action pairs from human knowledge or expert behavior - as heuristic guidance for exploration, incorporating them through potential-based shaping. However, this approach relies heavily on extensive human input, which can be difficult to obtain in complex environments. To mitigate this, fully autonomous methods have been proposed. 
For example, \citet{li_automatic} derive potential functions from offline datasets using causal state-value upper bounds under confounded data, and embed them into online learning for exploration. Other approaches automate shaping through meta-learning \citep{zou2019reward} or self-adaptive schemes that adjust potentials over time \citep{ma2025highly}.

\paragraph{Fine-tuning} involves starting from a pre-trained model (which has already learned general patterns from a large, broad dataset) and further training it on a specific task or domain \citep{zheng2025learning, li_2023}. It’s a key technique in \emph{transfer learning} \citep{PanYang2010,zhuang2020comprehensive}, where one uses prior experience in related tasks to accelerate learning on a new one. 



\section{Statistical setting}\label{sec:setting}
\label{sec:setting}

\paragraph{Markov decision process} We consider a finite-horizon episodic Markov Decision Process (MDP) defined by the tuple $\mathcal{M} {=} \crl{\mathcal{S}, \mathcal{A}, H, \crl{P_h}_{h=1}^{H}, \crl{r_h}_{h=1}^{H}, \rho}$, where $\mathcal{S}$ is the state space, $\mathcal{A}$ is the action space, $H$ is the horizon, $P_h(s'|s,a)$ is the transition probability at step $h$, $r_h(s,a) \in [0,1]$ is the deterministic known reward, and $\rho$ is the initial state distribution.

A policy $\pi {=} \{\pi_h\}_{h=1}^H$ is a sequence of functions where $\pi_h: \mathcal{S} \to \Delta(\mathcal{A})$. Here, for any set $\mathcal{X}$ equipped with some $\sigma$-field, $\Delta(\mathcal{X})$ denotes the set of probability distributions over $\mathcal{X}$.
The value function and action-value function for a policy $\pi$ are defined as:
\[
V_h^\pi(s) {=} \mathbb{E}^\pi \bigg[ \sum_{k=h}^H r_k(s_k, a_k) \Big| s_h=s \bigg],
\]    
and
\[
 Q_h^\pi(s,a) {=} r_h(s,a) + \mathbb{E}_{s' \sim P_h(\cdot|s,a)} \brk{V_{h+1}^\pi(s')}. 
\]
The optimal policy is denoted $\pi^\star$, with corresponding value functions $V_h^\star$ and $Q_h^\star$.
We place ourselves in the tabular setting, meaning that the state space $\mathcal{S}$ and action space $\mathcal{A}$ are finite sets. We denote the total number of states by $S:{=} |\mathcal{S}|$ and the number of actions by $A \coloneq |\mathcal{A}|$.

Furthermore, to simplify our analysis, we adopt the standard assumption (see \cite{xu_fine-grained_2021, zhang_provably_2021}) that the MDP is layered.\footnote{Any finite-horizon tabular MDP can be transformed into an equivalent layered MDP by augmenting the state space with the timestep, i.e., defining new states as $(s,h)$ for $s \in \mathcal{S}$ and $h \in \lbrace 1,\dots,H\rbrace$.}
The state space $\mathcal{S}$ is partitioned into $H$ disjoint subsets $\mathcal{S} {=} \bigcup_{h=1}^{H} \mathcal{S}_h$, with an additional terminal layer $\mathcal{S}_{H+1}{=}\{\perp\}$. Transition probability distributions $\{P_h\}_{h=1}^H$ are layer-dependent: for each $h \in \{1,\dots,H\}$, $P_h:\mathcal{S}_h\times\mathcal{A}\to\Delta(\mathcal{S}_{h+1})$ maps state–action pairs to distributions over the next layer. Rewards $\{r_h\}_{h=1}^H$ are similarly defined layer-wise, with $r_h:\mathcal{S}_h\times\mathcal{A}\to[0,1]$. The initial distribution $\rho\in\Delta(\mathcal{S}_1)$ is supported entirely on the first layer $\mathcal S_1$. We note $h(s)$ the layer $h$ that $s$ belongs to. Reinforcement learning algorithms typically take one of two forms: \emph{offline} and \emph{online}.


\paragraph{Offline algorithms (see \cite{li_settling_2024}) } In this setting, the algorithm does not have access to the MDP. Instead it receives a fixed batch of data $\mathcal{D}$ corresponding to $K$ trajectories, collected by a \emph{behavioural} policy $\pi^{\mathsf{b}}{=}\{\pi^{\mathsf{b}}_h\}_{1\leq h\leq H}$: 
\[\forall k \in [K], \quad 
\tau^{(k)}{=}\big(s^{(k)}_1,a^{(k)}_1,r^{(k)}_1,\dots,s^{(k)}_H,a^{(k)}_H,r^{(k)}_H\big), 
\]
where we have set $[K]:{=}\{1,\ldots,K\}$. This dataset is made up of i.i.d trajectories obtained by running $\pi^{\mathsf{b}}$ through the MDP both of which the learner may not have access to. Analysis usually packs this dataset in a single $\sigma$-family $\mathcal{G}_K {=} \sigma(\{\tau^{(k)}\}_{1\leq k \leq K}).$ The learner must extract estimates from these trajectories and produce a policy minimizing regret, with no possibility of adaptively querying new sample trajectories. These estimates are thus measurable with respect to $\mathcal{G}_K$.

\paragraph{Online algorithms (see \cite{zhang2024settling})} The algorithm interacts with the MDP for $T$ episodes, using a $t$-dependent policy $\pi^t$. Denoting each trajectory 
\[\forall t \in [T], \quad 
\tau_{t}{=}\big(s^{t}_1,a^{t}_1,r^{t}_1,\dots,s^{t}_H,a^{t}_H,r^{t}_H\big),
\]
the filtration $\mathcal J{=}(\mathcal{J}_t)_{t\in[T]}$ such that $\mathcal{J}_t {=} \sigma(\crl{\tau_{i}}_{i=1}^{t})$ encodes the information after $t$ episodes. Each policy $\pi^t$ used to explore episode $t$ is computed from the past $t{-}1$ episodes and is thus $\mathcal{J}_{t-1}$ measurable. This adaptivity distinguishes the online case from the offline one.

\section{Method: Envelope Shaping}\label{sec:method}

\subsection{From shaping functions to value envelopes}

Our algorithm builds on \cite{gupta_unpacking_2022} in which given \emph{shaping} functions $\widetilde{V}$ satisfying a \emph{$\beta$-sandwich property}
\[\widetilde{V}_h(s) \leq V_h^{\star}(s) \leq \beta \widetilde{V}_h(s),\]
and uses these bounds to rescale exploration bonuses and clip value estimates. This leads to regret bounds that scale with the size of an \emph{effective} state space, or state-action space, effectively pruning areas that the shaping function identifies as suboptimal. However, their analysis assumes that a sufficiently accurate shaping function is provided \emph{a priori}. Our algorithm allows for learned shaping functions as well as separate lower and upper bounds.

We propose \textbf{Envelope Shaping} in which we first perform an offline algorithm on a dataset $\mathcal{D}$ of $K$ offline trajectories that produces the following families of value \emph{envelopes}:
\begin{equation}
  \label{eq:envelopes}
  \crl{\lowv_{h}}_{h=1}^{H}, \quad \crl{\highv_{h}}_{h=1}^{H}, \quad \crl{\lowq_h}_{h=1}^{H}, \quad \crl{\highq_h}_{h=1}^{H}
\end{equation}
where $\highv_{h}$ and $\lowv_{h}$ are state value functions (i.e. functions of $s$), while $\lowq_h$ and $\highq_h$ are state-action value functions (i.e. functions of $(s,a)$) and the following holds for all $h \in [H]$ and $(s,a) \in \mathcal{S}_h \times \mathcal{A}$:
\[
\lowv_{h}(s) {=} \max_{a \in \mathcal{A}} \lowq_h(s,a), \quad \highv_{h}(s) {=} \max_{a \in \mathcal{A}} \highq_h(s,a)
\]

Our main working assumption is the following:

\begin{assumption} \label{assumption:sandwich} There exists $\delta>0$ and envelopes as in \eqref{eq:envelopes} such that defining the event $\mathcal{E}^{\mathrm{Off}}_\delta$ by  for all $h \in [H]$ and $(s,a) \in \mathcal{S}_h \times \mathcal{A}$, 
  \begin{equation}
  \label{eq:v_sandwich}
 \lowq_h(s,a) \leq Q_h^\star(s,a) \leq \highq_h(s,a)
\end{equation}
\begin{equation}
  \label{eq:q_sandwich}
 \lowv_{h}(s) \leq V_h^\star(s) \leq \highv_{h}(s), 
\end{equation}
we have $\mathbb{P}(\mathcal{E}^{\mathrm{Off}}_\delta) \geq 1-\delta$.
\end{assumption}

Our formalism departs from the single $\beta$-sandwich formulation used in \cite{gupta_unpacking_2022} and instead uses a pair of separate bounding functions $(\lowv_{h} , \highv_{h})$. This setting strictly generalizes the $\beta$-model as any $(\widetilde{V}, \beta)$ yields $\lowv_{h} {=} \widetilde{V}_h$ and $\highv_{h} {=} \beta\widetilde{V}_h$. However our formulation is more flexible as it decouples the lower and upper bound allowing for bounds obtained from \emph{separate independent procedures} and possibly even asymmetrically surrounding $V_h^\star$. It also allows for per layer bounds whereas $\beta$ is uniform over layers $h$.

\begin{remark}
    Our shaping functions are $\mathcal{G}_K$-measurable random variables. For our results, it suffices that assumptions \eqref{eq:v_sandwich} and \eqref{eq:q_sandwich} hold with high probability and that their computation is statistically independent of the MDP noise during the online phase, ensuring concentration inequalities apply. At episode $t$, step $h$, we then use an empirical Bernstein bound conditional on the $n$-sample $\crl{S_i'}_{i=1}^n$, i.e., $\highv_{h+1}(S_i') \mid (S_h^t{=}s, A_h^t{=}a)$.
\end{remark}

Our filtration \(\mathcal{J}\) separates the offline phase from the online phase, meaning that the fact that the shaping functions are $\mathcal{G}_K$-measurable guarantees this independence.

\subsection{Envelope-based Algorithm}

The algorithm we propose is a variant of the UCBVI algorithm from \cite{azar_minimax_2017} adapted to use the shaping functions $\highv_{h}$ and $\lowv_{h}$ to scale exploration bonuses and clip value estimates. At episode $t$ of our online phase, estimates will be $\mathcal{F}_{t-1}$-measurable where: 
\[\mathcal{F}_t {=} \sigma \bigl ( \mathcal{G}_{K} \cup \mathcal{J}_t \bigr ) \]

\paragraph{UCBVI \cite{azar_minimax_2017}.} UCBVI is a model-based algorithm that, at each episode $t$ and step $h$, forms empirical transition estimates $\widehat{P}^t_h(\cdot\mid s,a){=}(N_h^t(s,a,\cdot)/N_h^t(s,a))$ 
where $N_h^t(s,a)$ counts past occurrences of $(s,a,h)$. It then computes optimistic value estimates $\widehat{V}_h^t, \widehat{Q}_h^t$ by adding exploration bonuses - via concentration inequalities - to the rewards and running Value Iteration  with $\widehat{P}^t_h$. The exploration policy $\pi^t$ is \emph{greedy} w.r.t. $\widehat{Q}_h^t$.

\paragraph{Value Clipping.}
The clipping step remains unchanged from the modification to \textbf{UCBVI} applied in \cite{gupta_unpacking_2022}. We examine both ways of clipping value estimates: $V$-shaping and $Q$-shaping where the respective updates at episode $t$ and step $h$ are given by
\begin{align*}
\widehat{V}_h^t(s) & \gets \mathrm{clip} \left \lbrace \max_{a\in\mathcal{A}} \widehat{Q}_h^t(s,a),  \highv_h(s) \right \rbrace, \\
\widehat{Q}_h^t(s,a) & \gets \mathrm{clip} \left \lbrace r_h(s,a) + \widehat{P}_t  \widehat{V}_{h+1}^t + b_h^{t}(s,a),  \highq_h(s,a) \right \rbrace.
\end{align*}

Define 
\begin{equation}
D_h(s) \coloneqq \highv_{h}(s) - \lowv_{h}(s)
\end{equation}
\begin{equation}
    M_h(s) \coloneqq \dfrac{1}{2} (\highv_{h}(s) + \lowv_{h}(s))
\end{equation}
\begin{equation}
R_{h}\coloneqq\max_{s'}\highv_{h}(s')-\min_{s'}\lowv_{h}(s')
\end{equation}

\paragraph{Bonus Scaling}
Aside from the offline learning of $(\lowv, \highv)$, this is where the main modification arises. The bonus used at online episode $t$ and step $h$ is defined as:
\begin{equation}
    \bont(s,a) \coloneqq
c_1 \sigma_{h+1}^t(s,a) \sqrt{\frac{ L}{N_h^t(s,a)}} + 
c_2 \frac{R_{h+1} L}{N_h^t(s,a)},
\end{equation}
with
\begin{align*}
\sigma_{h+1}^t(s,a) & \coloneqq \sqrt{\mathrm{Var}_{s'\sim\widehat{P}_t(\cdot|s,a)} (M_{h+1}(s') )}\\ 
& + \dfrac12\sqrt{\mathbb{E}_{s'\sim\widehat{P}_t(\cdot|s,a)} [D_{h+1}(s')^2 ]},
\end{align*}
where $L$ is a logarithmic term, $c_1$ and $c_2$ are constants.
In \textbf{UCBVI}, using the empirical Bernstein bound from  \cite{maurer_empirical_2009}, the bonus must upper bound the right hand term below:
\[ |(\widehat{P}^t_h - P_h^\star, V^\star_{h+1})| \leq \sqrt{\frac{\mathrm{Var}(V^\star_{h+1})L}{N_h^t(s,a)}} + 
c_2 \frac{R_{h+1} L}{N_h^t(s,a)}.\]
In our case we bound the variance of the optimal value function with $\sigma_{h+1}^t$. The intuition of this bound is that as the width of the bounding interval tightens ($D_{h+1} \to 0$) we get $M_{h+1} \to V_{h+1}^\star$ recovering the original bound (without ever having access to $V^\star$).

\begin{algorithm}[h!]
\caption{$Q$-shaping}
\label{alg:q_shaping}
\begin{algorithmic}[1]
\Require Bounds $\lowv, \highv$; confidence level $\delta$
\State Initialize counts $N \gets 0$; $\widehat{V}_{H+1}^{t} \gets 0$;
\For{$t \in  [1,2,\dots, T]$}
    \State Build $\widehat{P}_t$ from $N$
    \For{$h {=} H, \dots, 1$}
        \ForAll{$(s,a)$}
            \State Compute width-based bonus $b_h^{t}(s,a)$
            \State $\widehat{Q}_h^t(s,a) \gets \mathrm{clip} \Big\{ r_h(s,a) {+} \langle \widehat{P}^t_{h,s,a}, \widehat{V}_{h+1}^t \rangle$
            \Statex \hspace{10.0em}${+} \bont(s,a), \highq_h(s,a) \Big\}$
        \EndFor
        \ForAll{$s$}
            \State $\widehat{V}_h^t(s) \gets \max_{a} \widehat{Q}_h^t(s,a)$
        \EndFor
    \EndFor
    \For{$h {=} 1, \dots, H$}
        \State $a_{t,h} \in \arg\max_{a} \widehat{Q}_h^t(s_{t,h}, a)$
        \State Take $a_{t,h}$, observe $(r_{t,h}, s_{t,h+1})$, update $N$
    \EndFor
\EndFor
\end{algorithmic}
\end{algorithm}

\subsection{Shaping Online from Offline data}

We propose a \emph{stand-in} procedure to meet Assumption~\ref{assumption:sandwich}, by learning the envelope functions via standard offline Value Iteration . Unlike standard offline RL, which replaces \textbf{optimism} with \textbf{pessimism} to derive value lower bounds, our goal is not regret minimization but envelope construction for the online phase. We therefore combine optimism and pessimism, assuming \emph{favorable} conditions:

\begin{assumption} \label{assumption:fc} States are reachable from the initial distribution $\rho$, and the behavior policy $\pi^{\mathsf{b}}$ provides full coverage of the MDP, formally: 
\begin{equation}
    \forall (s,h) \in \mathcal{S}\times[H], \mathbb{P}^{\pi^{\mathsf{b}}}(s_h = s) \geq d^{\mathsf b}_{\min} > 0
\end{equation}
\end{assumption}


Given $K$ trajectories in a dataset $\mathcal{D}$, the offline algorithm computes Bernstein-based bonuses $\boff$, ensuring the optimal value functions are bounded with probability at least $1-\delta$, where $\delta$ is a parameter of $\boff$. The complete proof is deferred to the appendix.

\begin{remark}[Full coverage vs concentrability]
Standard offline RL analysis often relies on a weaker single-policy concentrability assumption ($C^\star$). However, such guarantees are generally limited to the pure offline setting or hybrid settings that assume the learner retains access to the offline dataset during the online phase. For instance, \cite{li_2023} employ a three-stage (offline-online-offline) protocol, \cite{pattanaik_2025} allow simultaneous access to offline samples, and \cite{xie_policy_2021} use offline data to solve the initial steps of the horizon (up to $h^{\star}$) and online for the rest. In contrast, our strict two-stage framework calculates envelopes once and discards the raw data to satisfy privacy constraints, passing only the bounds to the online agent. 
Remark~\ref{rem:paireff} further clarifies our rationale for adopting Assumption \ref{assumption:fc}.
\end{remark}

\begin{algorithm}[h!]
\caption{Upper-Lower Offline Reinforcement Learning}
\label{alg:offline}
\begin{algorithmic}[1]
\Require $\mathcal{D}{=}\{(s_1^{(i)}, a_1^{(i)}, r_1^{(i)}, \dots,s_H^{(i)}, a_H^{(i)}, r_H^{(i)})\}_{i=1}^K$ (offline dataset)
\State Randomly split $\mathcal{D}$ into $\{\mathcal{D}_{h}\}_{h=1}^H$ with $m_h=|\mathcal{D}^{(h)}| \ge \left \lfloor K/H \right \rfloor$.
\State Let $N_{h}(s,a)$ and $N_{h}(s,a,s')$ be the visitation count of $(s,a)$ and $(s,a,s')$ at step $h$ within $\mathcal{D}_h$.
\State Set $\olowv_{H+1} {=} \ohighv_{H+1} = 0$
\For{$h {=} H, \dots, 1$}
    \ForAll{$(s,a)$}
        \State Compute bonus $\boff(s,a)$
        \State $
\ohighq_h(s,a){=}r_h(s,a){+}\langle\widehat{P}_{h,s,a},\ohighv_{h+1}\rangle{+}\boff(s,a)$
        \State $\olowq_h(s,a){=}r_h(s,a){+}\langle\widehat{P}_{h,s,a},\olowv_{h+1}\rangle{-}\boff(s,a)$
    \EndFor
    \ForAll{$s$}
        \State $\ohighv_h(s)=\max_{a\in\mathcal{A}}\ohighq_h(s,a)$
        \State $\olowv_h(s)=\max_{a\in\mathcal{A}}\olowq_h(s,a)$
    \EndFor
\EndFor
\end{algorithmic}
\end{algorithm}

\section{Theoretical bounds}\label{sec:bounds}

\subsection{Regret bound for $Q$-shaping}

\begin{figure}[h!]
\centering
\includegraphics[width=\columnwidth]{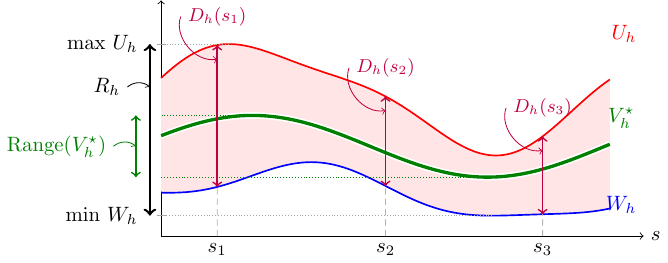}
\caption{Upper and lower envelopes $\highv_{h}$ and $\lowv_{h}$ around $V_h^\star$. The vertical indicators $D_h(s_i)$ show the varying gap width between the envelopes at specific points.}
\label{fig:my_envelopes}
\end{figure}

Our analysis of the online algorithm guided by these offline learned bounds yields a regret decomposition with two novel factors: the maximum width of the bounding interval and the maximum range of the interval, respectively noted as:
\begin{equation}\label{eq:D_R_max}
D^{\max}\coloneq\max_h \sup_s D_h(s) \quad \text{and} \quad R^{\max}\coloneq\max_h R_h.
\end{equation}

As $D_h$ is a width of a point wise envelope $(\lowv, \highv)$, this will depend on the quality of the envelope and can theoretically be very small. On the other hand $R_h$ inherits a lower bound from the MDP as illustrated in Figures \ref{fig:my_envelopes} and \ref{fig:D_H}:
\begin{equation*}
    \mathrm{Range}(V_h^{\star}) \coloneq \sup_{s} V_h^{\star}(s) - \inf_{s} V_h^{\star}(s)
\end{equation*}
This formulation yields several advantages over the $\beta \widetilde{V}$ approach. First, $R^\mathrm{max}_h$ is analogous to $\beta \widetilde{V}^\mathrm{max}_h$ but is strictly smaller, as it leverages the lower bound $\lowv_{h}$ to tighten the effective range. Second, when learning the envelope with our offline algorithm: under suitable offline conditions (e.g, Assumption \ref{assumption:fc}), $D^\mathrm{max}_h(s) \to 0$ as the offline dataset size $K$ grows, causing uncertainty-dependent terms to vanish and further accelerating convergence. This directly links the offline data to the online sample complexity improvements, providing a formal bridge between offline pre-training and online fine-tuning. 

For the $V$-shaping, the analysis of \cite{gupta_unpacking_2022} incorporates pseudo sub-optimality sets ($\PPS$, $\BPS$ and $\PS$ defined in Section~\ref{sec:V_shaping}). The $Q$-shaping analysis does not require these definitions. For readability we focus on the $Q$-shaping version of the algorithm for the clarity of its regret. The $V$-shaping regret bound can be found in Section~\ref{sec:V_shaping}. 
\begin{figure}[H]
\centering
\includegraphics[width=0.65\columnwidth]{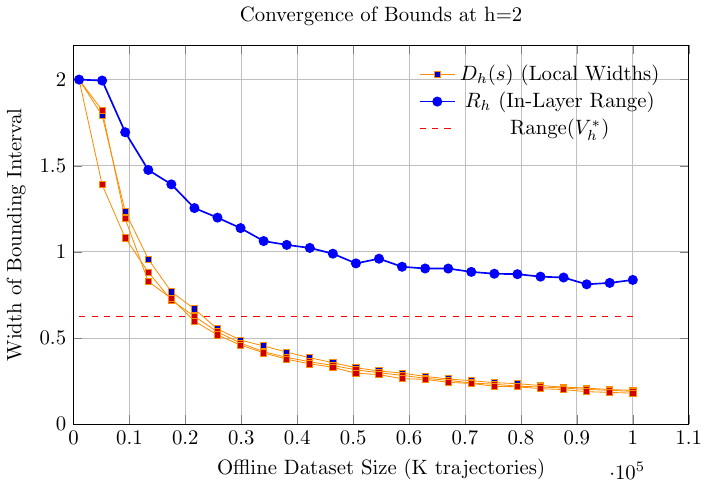}
\caption{Illustration of how $D_h {\to} 0$ and  $R_h{\to} \mathrm{Range}(V^{\star}_h)$ as $K {\to} \infty$. Here there are 3 actions, $H{=}4$ with 3 states per layer. The plot represents $D_2(s)$ for the states in layer $h=2$ with respect to $R_h$ for varying $K$.}
\label{fig:D_H}
\end{figure}
\begin{theorem}
  \label{thm:q_shaping}
Suppose that Assumption~\ref{assumption:sandwich} holds with some constant $\delta>0$. Then, with probability at least $1-4\delta$, the regret of Algorithm~\ref{alg:q_shaping} after $T$ episodes satisfies:
\begin{align*}
\mathrm{Regret}(T) 
\le
R^{\max}\widetilde{\mathcal O}\Big(\sqrt{TH|\mathrm{PairEff}|}+|\mathrm{PairEff}|+\sqrt{T}\Big) \\ +
D^{\max}\widetilde{\mathcal O}\Big(\sqrt{TH|\mathrm{PairEff}|}+|\mathcal{S}|H|\mathrm{PairEff}| +|\mathcal{S}|H\sqrt{T}\Big)
\end{align*}
where $\widetilde{\mathcal O}$ hides polylog factors in 
$T,|\mathcal{S}|,|\mathcal{A}|,H,$ and $1/\delta$, where $R^{\max}$ and $D^{\max}$ are defined by \eqref{eq:D_R_max} respectively, and
\begin{equation}\label{eq:paireff}
    \mathrm{PairEff} \coloneq \left \lbrace (s,a) \in \mathcal{S} \times \mathcal{A} ~ | ~ \highq_{h(s)}(s,a) \geq V^{\star}_{h(s)}(s) \right \rbrace
\end{equation} 
\end{theorem}

Using $\ohighq_h$ and $\olowq_h$ produced by Algorithm \ref{alg:offline} as $\highq_h$ and $\lowq_h$ in Algorithm~\ref{alg:q_shaping} yields the following regret bound:

\begin{theorem}
  \label{thm:offline_q_shaping}
  With probability at least $1-4\delta$, the regret  after $T$ episodes of running Algorithm~\ref{alg:q_shaping} using the output of Algorithm~\ref{alg:offline} as an envelope satisfies:
\begin{flalign*}
& \mathrm{Regret}(T) 
\le
\widetilde{\mathcal O} \Bigg(R^{\max}\Big(\sqrt{TH|\mathrm{PairEff}|}+|\mathrm{PairEff}|+\sqrt{T}\Big) && \\ 
& {+} \sqrt{\dfrac{H^{5}}{K d^{\mathsf b}_{\min}}}\Big(\sqrt{TH|\mathrm{PairEff}|}+|\mathcal{S}|H|\mathrm{PairEff}| +|\mathcal{S}|H\sqrt{T}\Big) && \\
& {+} \dfrac{H^{3}}{(K d^{\mathsf b}_{\min})}\Big(\sqrt{TH|\mathrm{PairEff}|}+|\mathcal{S}|H|\mathrm{PairEff}| +|\mathcal{S}|H\sqrt{T}\Big) \Bigg).&&
\end{flalign*}
\end{theorem}
\begin{remark}[$\mathrm{PairEff}$ and Assumption \ref{assumption:fc}]
\label{rem:paireff}
   Note that under the weaker partial coverage assumption (with concentrability coefficient $C^\star$), tight envelopes would be guaranteed only along the optimal policy $\pi^\star$. \textbf{This would leave suboptimal actions effectively uncovered}, resulting in \emph{vacuous upper bounds} $\highq_h(s,a) \approx H$ in \eqref{eq:paireff}. Since $H \ge V^\star_h(s)$ and allowing for these pairs to remain in $\mathrm{PairEff}$. In contrast, Assumption \ref{assumption:fc} tightens $\highq$ uniformly, allowing the algorithm to \textbf{actively identify and discard suboptimal actions}, thereby shrinking $\mathrm{PairEff}$ and accelerating online exploration. A consequence of this is that poor behavior policies (e.g., far from $\pi^{\star}$) such as uniform policies still provide provable acceleration.
\end{remark}

\subsection{Proof Sketch}

The proof sketch follows \cite{gupta_unpacking_2022}, but introduces $R_h$ and $D_h$ when adjusting bonus scaling, decomposing regret, and bounding the online bonus via a martingale argument. Finally, $D^{\max}$ is controlled by the offline dataset.

Our main \textbf{technical novelty}, and the key point where our proof departs from prior shaping analyses, is a \textbf{conditional version of empirical Bernstein} inequality that remains valid when the bounds necessary to apply Bernstein \cite{maurer_empirical_2009} are random (learned offline), but measurable with respect to the offline sigma-algebra and therefore independent of the online transitions. This result may be of independent interest. In our setting, it is precisely what allows us to rigorously control the online bonus on the event where the offline envelopes are valid, since these random bounds coincide exactly with the learned envelopes. At first sight, one cannot directly plug these into the standard inequality: the high-probability event “the offline envelopes are valid” is not independent of the random variables that appear in the online martingale differences, 
so conditioning on it naively could invalidate the concentration step. We resolve this by making the measurability structure explicit, this technical result is given in Section~\ref{sec:cond_bernstein} of the Appendix.

\paragraph{Bonus Scaling.} Contrary to \cite{gupta_unpacking_2022}, our concentration argument relies on classical variance aware results from \cite{maurer_empirical_2009} and does not require anytime formulations. In practice, we apply our conditional empirical Bernstein inequality to:
 \[Z_i = \frac{V_{h+1}^{\star}(S_i) - \min_{s'} \lowv_{h+1}(s')}{R_{h+1}} \in [0,1]\]
 which is where $R_{h+1}$ is introduced, later showing up as $R^{\max}$ in Theorem~\ref{thm:q_shaping}. The empirical variance at step $t$ of $V^{\star}_{h+1}$ can then be bounded by $\sigma^t_{h+1}$ which contains both $\mathrm{Var}_{s'\sim\widehat{P}_t(\cdot|s,a)} (M_{h+1}(s') )$ and $\mathbb{E}_{s'\sim\widehat{P}_t(\cdot|s,a)} [D_{h+1}(s')^2 ]$.
 
\paragraph{Regret Decomposition.} Standard techniques (see \cite{agarwal2019reinforcement}, Lemma 7.10) for bounding the regret derive expressions such as
\begin{align*}
V^\star_{1}(s)-V^{\pi^t}_{1}(s)
 {\le} C \cdot
\sum_{h=1}^{H}\mathbb{E}^{\pi^t} \left[2\bont(s_h,a_h) + \nu_h^t(s_h,a_h) \right],
\end{align*}
by combining the following recursion
\begin{align*}
\Delta_h(s_h) &\le 2\bont(s_h,a_h) + \xi_h^{t}(s_h,a_h) \\
&+ \mathbb{E} \left[\Delta_{h+1}(s_{h+1})\mid s_h,a_h\right]
\end{align*}
where $\xi_h^{t}(s,a){=}\left \langle \widehat{P}_{h,s,a}^{t} -P_{h,s,a}^{\star}, \widehat{V}_{h+1}^t - V_{h+1}^\star \right \rangle$ and $\Delta_h(s_h){=} \widehat{V}^t_h(s_h){-} V^{\pi_t}_h(s_h)$ together with the inequality 
\[
\xi_h^{t}(s,a) \le \frac{\mathbb{E}_{s' \sim P_{h,s,a}^{\star}} \left[ \Delta_{h+1}(s') \right]}{H} + \nu_h^{t}(s,a).
\]
This inequality holds when the expression of $\nu_h^{t}$ is well chosen and specifically depends on a uniform bound on  $(\widehat{V}_{h+1}^t - V_{h+1}^\star)$. Using 
\[ 
\widehat{V}_{h+1}^t(s') - V_{h+1}^\star(s') \le \highv_{h+1}(s') - \lowv_{h+1}(s') \le D_{h+1}^{\max},
\]
allows for explicit integration of $D_{h+1}^{\max}$ in the regret.
\paragraph{Martingale Argument} Going from an expectation to a high probability bound is classically done by applying martingale arguments \cite{hoeffding_probability_1963, azuma_weighted_1967}  requiring a uniform bound on specific martingale difference sequences. Using
\begin{align*}
\bont(s,a){\le} 
c_1\left(\tfrac12R_{h+1}{+} \tfrac12 D^{\max}_{h+1}\right)\sqrt{\frac{L}{N_h^t(s,a)}}{+} 
c_2\frac{R_{h+1}L}{N_h^t(s,a)},
\end{align*}
we can maintain the presence of $D^{\max}$ and $R_h$ in the high probability bound.

\paragraph{Offline to Online} The final step bounds $D^{\max}$ in terms of the dataset size $K$ and the minimum coverage $d^{\mathsf b}_{\min}$ of $\pi^{\mathsf b}$. For any $h$, $D_h$ is controlled by the expected future bonuses. Uniformly lower bounding $N_h^t$ with a Chernoff bound then yields a uniform upper bound on these bonuses, and thus on $D_h^{\max}$.



\section{$V$-shaping}\label{sec:V_shaping}
We include here the corresponding result for applying method to the $V$-shaping algorithm.

\begin{algorithm}[h!] 
\caption{$V$-shaping}
\label{alg:v_shaping}
\begin{algorithmic}[1]
\Require Bounds $\lowv, \highv$; confidence level $\delta$
\State Initialize counts $N \gets 0$; $\widehat{V}_{H+1}^{t} \gets 0$
\For{$t = 1,2,\dots$}
    \State Build $\widehat{P}_t$ from $N$
    \For{$h = H, \dots, 1$}
        \ForAll{$(s,a)$}
            \State Compute width-based bonus $b_h^{t}(s,a)$
            \State $\widehat{Q}_h^t(s,a) {\gets} r_h(s,a) {+} \langle \widehat{P}^t_{h,s,a},\widehat{V}_{h+1}^t \rangle {+} b_h^{t}(s,a)$
        \EndFor
        \ForAll{$s$}
            \State $\widehat{V}_h^t(s) \gets \min\left\{ \max_{a} \widehat{Q}_h^t(s,a),  \highv_{h}(s) \right\}$
        \EndFor
    \EndFor
    \For{$h = 1, \dots, H$}
        \State $a_{t,h} \in \arg\max_{a} \widehat{Q}_h^t(s_{t,h}, a)$
        \State Take $a_{t,h}$, observe $(r_{t,h}, s_{t,h+1})$, update $N$
    \EndFor
\EndFor
\end{algorithmic}
\end{algorithm}
The regret is more intricate and, following \cite{gupta_unpacking_2022}, can be interpreted by defining $Q_h^{\highv}(s,a)=r_h(s,a)+\mathbb{E}_{s'}[\highv_{h+1}(s')]$ and, for $\Delta>0$, the sets:
\begin{align*}
\PS
&=\Big\{(s,a,h): Q_h^{\highv}(s,a)\le V_h^\star(s)-\Delta\Big\},\\
\PPS 
&= \left\{\, s \in \mathcal{S} \;\middle|\; 
\substack{\text{all feasible paths from } \rho \text{ to } s \\ 
\text{intersect } \PS} \right\},\\
\BPS
&= \left \lbrace (s,a,h) \in \PS ~ \vert ~ s \in \PPS  \right \rbrace.
\end{align*}
Our contribution is the presence of $D^{\max}$ and $R^{\max}$ in the bound, and that $\highv_{h}$ and $\lowv_{h}$ stem from our offline protocol.
The broad strokes of the proof are that any pair $(s,a)$ in $\PS$ is eliminated after $O(\mathrm{poly}(R_h/\Delta^2))$
visits. Furthermore, states that are reachable only through such pairs constitute $\PPS$ and require only
limited exploration ($\BPS$). Our adaptation preserves this structure, replacing $(\widetilde{V},\beta\widetilde{V})$ with $(\lowv,\highv)$.

\begin{theorem}[$V$-shaping Regret]
    \label{thm:v_shaping_regret}
    For all $\Delta > 0$, the regret of Algorithm \ref{alg:v_shaping} can be bounded with high probability by:
    \begin{align*}
    \mathrm{Regret}(T) & \leq \widetilde O\Big(
(R^{\max}{+}D^{\max})\sqrt{T H |\mathcal{S}\setminus \PPS| |\mathcal{A}|} \\
&{+}\sqrt{T}(R^{\max}{+}|\mathcal{S}| H D^{\max}) \\
&{+} R^{\max} |\mathcal{S}\setminus \PPS| |\mathcal{A}|{+} D^{\max} |\mathcal{S}\setminus \PPS|^2 |\mathcal{A}| H\\
&{+} \min\{\Theta_{1}(\Delta),\Theta_{2}(\Delta)\}
\Big)
\end{align*}
where
\begin{align*}
  \Theta_{1}(\Delta)& =\dfrac{R^{\max}+D^{\max}}{\Delta}\sqrt{|\BPS||\mathcal{S}| |\mathcal{A}|H} \\
& + D^{\max}H|\mathcal{A}||\mathcal{S}^2| + R^{\max}|\mathcal{S}||\mathcal{A|},
\end{align*}
\[
\Theta_{2}(\Delta)=\dfrac{(1 + R^{\max})^2}{\Delta^2}|\BPS|(R^{\max}+|\mathcal{S}| H D^{\max})H.
\]
\end{theorem}

Substituting $\ohighv_h,\olowv_h$ from Algorithm~\ref{alg:offline} into Algorithm~\ref{alg:q_shaping} gives a bound depending on $K$ and $d^{\mathrm b}_{\min}$, analogous to the $Q$-shaping case, deferred to the appendix.

\section{Experiments}\label{sec:experiments}

\paragraph{Effect of $K$.} We compare our $Q$-shaping versions with learned value envelopes offline using $K$ trajectories against \textbf{Bernstein UCBVI}, the baseline algorithm without any offline shaping, serving as a reference for performance. Additionally we compare against a version of \textbf{Bernstein UCBVI} in which offline transitions have been directly added to the initial counts state-action counts, we name this method \textbf{Count-Initialized UCBVI}.

\begin{figure}[h!]
\centering
\includegraphics[width=0.70\linewidth]{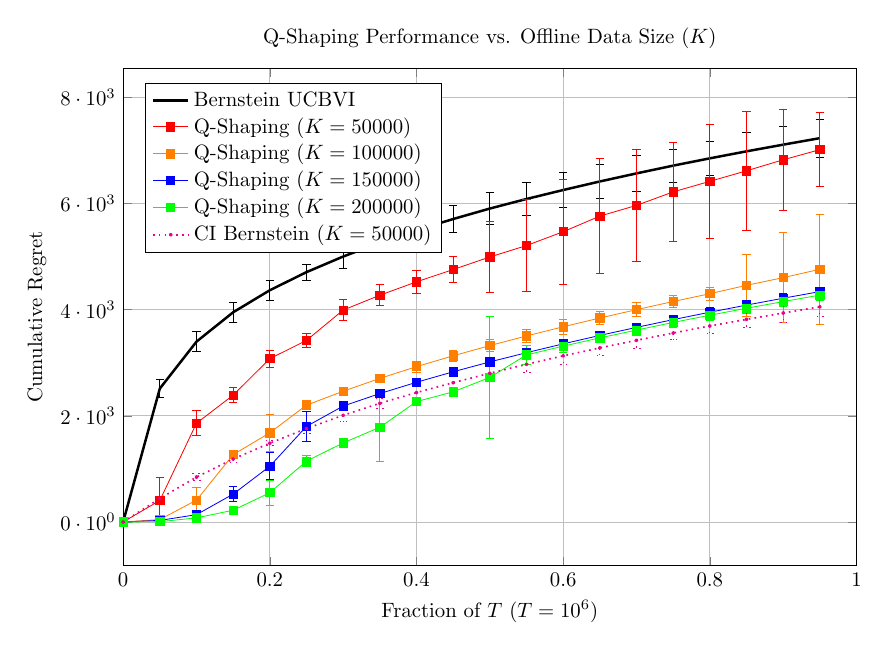}
\caption{$Q$-shaping with varying $K$. Layered MDP with $H {=} 10$ and $30$ states, with $3$ actions per state. Average over $10$ and std-bars over seeds.}
\label{fig:regret}
\end{figure}

Figure~\ref{fig:regret} shows us that as $K$ increases cumulative regret decreases, yielding significant gains over Bernstein UCBVI. Count-Initialized UCBVI is however more performant with respect to offline budget ($K$), though it necessitates knowledge of the offline dataset which is not the case of our method which respects privacy and is nevertheless competitive. As $K$ increases, cumulative regret decreases largely due to the shrinking $D^{\max}$ terms. While $R^{\max}$ also shrinks, it is bounded below by $\mathrm{Range}(V^\star)$. To study this effect, we parametrize the MDP to control $\mathrm{Range}(V^\star)$:

\paragraph{Influence of $R^{\max}$.}
To assess the effect of $R^{\max}$ on regret, we run experiments on layered tabular MDPs, whose structure constrains $\mathrm{Range}(V^\star)$ and thus $R^{\max}$. Specifically, we set $r_h(s,a) = 0$ for all steps $h < H$, concentrating the reward signal at the final transition. The terminal rewards $\{r_{H}(s,a)\}_{(s,a){\in} \mathcal{S}_{H}\times\mathcal{A}}$ are drawn from a range $[r_1, r_2] {\subseteq} [0,1]$. By backward induction, we can show $V_h^\star(s){\in}[r_1,r_2]$ for all $h,s$ as the Bellman optimality equation is $V_h^\star {=}\max_a (r_h + \langle P^{\star}_h,  V_{h+1}^\star\rangle)$ and:\\
$\bullet$ For $h=H$, $V_{H+1}^\star \equiv 0$, so $V_H^\star{=} \max_a r_H \in [r_1,r_2]$;
\vspace{4pt} \\
$\bullet$ For $h{<}H$, $r_h{=} 0$. If $V_{h+1}^\star{\in}[r_1,r_2]$, then its expectation $\langle P^{\star}_h,  V_{h+1}^\star\rangle$ is also in $[r_1,r_2]$, and so is $V_h^\star {=} \max_a \langle P^{\star}_h,  V_{h+1}^\star\rangle$.

Because the scale of the regret is dependent on the MDP instance, we normalize performance by measuring the \textbf{Relative Regret Improvement}, defined as $(\text{Regret}_{\text{Baseline}} - \text{Regret}_{\text{Algo}}) / \text{Regret}_{\text{Baseline}}$.

To compare our method to the $\beta \widetilde{V}$ shaping from \cite{gupta_unpacking_2022}, we experiment with shaping using only the upper bound, meaning setting $\lowv_{h}$ to $0$. This means replacing 
$\sigma_{h+1}^t$ with
\begin{align*}
\widetilde{\sigma}_{h+1}^t(s,a) 
&\coloneqq \dfrac{1}{2}\sqrt{\mathrm{Var}_{s'\sim\widehat{P}_t(\cdot|s,a)} (\highv_{h+1}(s') )}\\
&+ \dfrac12\sqrt{\mathbb{E}_{s'\sim\widehat{P}_t(\cdot|s,a)} [\highv_{h+1}(s')^2 ]}
\end{align*}
and replacing $R_{h+1}$ with $\max_{s'}\highv_{h+1}(s')$. This  is closely related to the way \cite{gupta_unpacking_2022} shape their bonus (using $\highv_{h}{=} \beta\widetilde{V}_h$). Since $\widetilde{\sigma}_{h+1}^t(s,a) {\leq} \sqrt{\mathbb{E}_{s'\sim\widehat{P}_t}[\highv_{h+1}(s')^2]}$, if the algorithm outperforms $\lowv_h=0$ then it outperforms \cite{gupta_unpacking_2022}. We call this \textbf{Upper-Bonus} Shaping. Our main novelty lies in computing envelopes and shaping the bonus with them, which we compare to this method without clipping i.e $V$-shaping refered to as \textbf{Full-Bonus}. As a baseline algorithm, we use Hoeffding UCBVI to provide a fair comparison against \cite{gupta_unpacking_2022} as their method is only competitive with respect to Hoeffding UCBVI.


 Figure~\ref{fig:sliding} depicts two setups that vary the range intervals $[r_1,r_2]$.  They both plot as a function of $x$ the Relative Regret Improvement after $T$ episodes with ranges:  $[1-x,1]$ for  \textbf{Expanding Range} and  $[x, x + w]$ for \textbf{Sliding Range} where $w {=} 0.1$ is the width of the range. Here, the MDP has $H{=}10$ and $20$ states per layer with $3$ actions. We use $K{=}6000$ for the offline phase and $T {=} 10^6$. The algorithms are run with $10$ seeds and averaged. These experiments are meant to accentuate the settings in which our method should have advantages over the prior Upper Bonus Shaping version:

\begin{figure}[H]
    \centering
\includegraphics[width=\columnwidth]{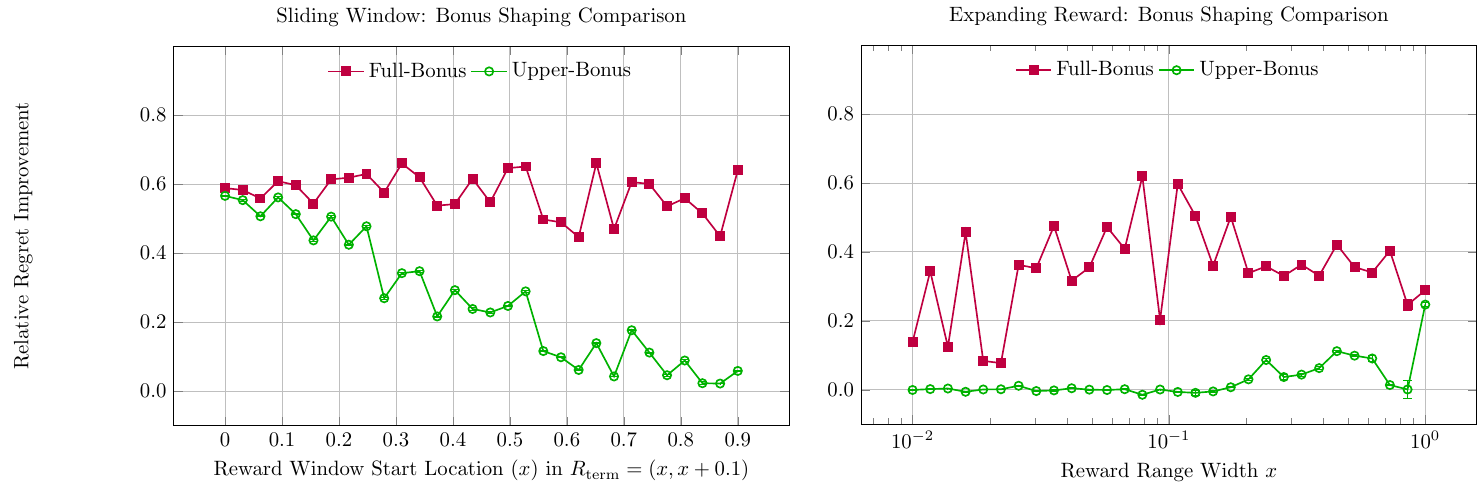}
\caption{Sliding Range (left) and Expanding Range (right)}
\label{fig:sliding}
\end{figure}



\begin{itemize}
    \item \textbf{Expanding Range:} When $x$ is low, $R^{\max}$ is likely to be low and the lower bound $\lowv$ can be very informative (being far from 0). In this setting \textbf{Upper Bonus Shaping} struggles to make any improvements over UCBVI as it only uses the upper bound on the value which is likely to be very close to $H$, bringing little information. When $x$ grows, the upper bound becomes informative. Over the entire range of $x$ our method succeeds in gaining an advantage.
    \item \textbf{Sliding Range:} When $x$ is close to 0, the lower bound in uninformative and the more $x$ slides the window to the right, the more the lower and upper bound bring information. Here again our method gains the upper hand.
\end{itemize}

\section{Discussion and conclusion}
\label{sec:discussion}


Our work develops a framework for integrating upper and lower bounds of the optimal value function into reinforcement learning algorithms. In contrast to prior work, our framework permits the confidence bounds to be random. This is made
possible by a careful analysis of the conditions under which the key concentration results
underlying classical online regret guarantees remain valid. A natural extension would be to learn these
envelopes in related but distinct MDPs, thereby entering a \emph{transfer learning} regime. In that
setting, one would expect guarantees similar to ours, augmented by a notion of distance between 
source and target MDPs. Finally, we modify only the bonus function while retaining the clipping
mechanisms of prior work. However, more refined clipping rules incorporating both lower and upper bounds (for instance, eliminating actions whenever $\underline{Q}_h(s,a) \le \min_{a'}\overline{Q}_h(s,a') - \Delta$) could further reduce the set of state–action pairs requiring online
exploration.



\bibliography{Papier}

\begin{thebibliography}{}

\bibitem[Agarwal et~al., 2019]{agarwal2019reinforcement}
Agarwal, A., Jiang, N., and Kakade, S.~M. (2019).
\newblock Reinforcement learning: Theory and algorithms.

\bibitem[Ajay et~al., 2020]{ajay2020opal}
Ajay, A., Kumar, A., Agrawal, P., Levine, S., and Nachum, O. (2020).
\newblock Opal: Offline primitive discovery for accelerating offline reinforcement learning.
\newblock {\em arXiv preprint arXiv:2010.13611}.

\bibitem[Azar et~al., 2017]{azar_minimax_2017}
Azar, M.~G., Osband, I., and Munos, R. (2017).
\newblock Minimax {Regret} {Bounds} for {Reinforcement} {Learning}.
\newblock In {\em Proceedings of the 34th {International} {Conference} on {Machine} {Learning}}, pages 263--272. PMLR.
\newblock ISSN: 2640-3498.

\bibitem[Azuma, 1967]{azuma_weighted_1967}
Azuma, K. (1967).
\newblock Weighted sums of certain dependent random variables.
\newblock {\em Tohoku Mathematical Journal}, 19(3).

\bibitem[Brys et~al., 2015]{brys2015reinforcement}
Brys, T., Harutyunyan, A., Suay, H.~B., Chernova, S., Taylor, M.~E., and Now{\'e}, A. (2015).
\newblock Reinforcement learning from demonstration through shaping.
\newblock In {\em IJCAI}, pages 3352--3358.

\bibitem[Cheng et~al., 2021]{cheng2021heuristic}
Cheng, C.-A., Kolobov, A., and Swaminathan, A. (2021).
\newblock Heuristic-guided reinforcement learning.
\newblock {\em Advances in Neural Information Processing Systems}, 34:13550--13563.

\bibitem[Chernoff, 1952]{chernoff_measure_1952}
Chernoff, H. (1952).
\newblock A {Measure} of {Asymptotic} {Efficiency} for {Tests} of a {Hypothesis} {Based} on the sum of {Observations}.
\newblock {\em The Annals of Mathematical Statistics}, 23(4):493--507.

\bibitem[Dann et~al., 2021]{dann2021beyond}
Dann, C., Marinov, T.~V., Mohri, M., and Zimmert, J. (2021).
\newblock Beyond value-function gaps: Improved instance-dependent regret bounds for episodic reinforcement learning.
\newblock {\em Advances in Neural Information Processing Systems}, 34:1--12.

\bibitem[Golowich and Moitra, 2022]{golowich2022can}
Golowich, N. and Moitra, A. (2022).
\newblock Can q-learning be improved with advice?
\newblock In {\em Conference on Learning Theory}, pages 4548--4619. PMLR.

\bibitem[Gupta et~al., 2019]{gupta2019relay}
Gupta, A., Kumar, V., Lynch, C., Levine, S., and Hausman, K. (2019).
\newblock Relay policy learning: Solving long-horizon tasks via imitation and reinforcement learning.
\newblock {\em arXiv preprint arXiv:1910.11956}.

\bibitem[Gupta et~al., 2022]{gupta_unpacking_2022}
Gupta, A., Pacchiano, A., Zhai, Y., Kakade, S.~M., and Levine, S. (2022).
\newblock Unpacking {Reward} {Shaping}: {Understanding} the {Benefits} of {Reward} {Engineering} on {Sample} {Complexity}.

\bibitem[Hoeffding, 1963]{hoeffding_probability_1963}
Hoeffding, W. (1963).
\newblock Probability {Inequalities} for {Sums} of {Bounded} {Random} {Variables}.
\newblock {\em Journal of the American Statistical Association}, 58(301):13--30.

\bibitem[Huang et~al., 2025]{huang_augmenting_2025}
Huang, R., Li, D., Shi, C., Shen, C., and Yang, J. (2025).
\newblock Augmenting online {RL} with offline data is all you need: a unified hybrid {RL} algorithm design and analysis.
\newblock In {\em Proceedings of the {Forty}-{First} {Conference} on {Uncertainty} in {Artificial} {Intelligence}}, volume 286 of {\em {UAI} '25}, pages 1745--1767, Rio de Janeiro, Brazil. JMLR.org.

\bibitem[Jin et~al., 2018]{jin2018q}
Jin, C., Allen-Zhu, Z., Bubeck, S., and Jordan, M.~I. (2018).
\newblock Is q-learning provably efficient?
\newblock {\em Advances in neural information processing systems}, 31.

\bibitem[Kakade and Langford, 2002]{kakade2002approximately}
Kakade, S. and Langford, J. (2002).
\newblock Approximately optimal approximate reinforcement learning.
\newblock In {\em Proceedings of the nineteenth international conference on machine learning}, pages 267--274.

\bibitem[Kang et~al., 2018]{kang2018policy}
Kang, B., Jie, Z., and Feng, J. (2018).
\newblock Policy optimization with demonstrations.
\newblock In {\em International conference on machine learning}, pages 2469--2478. PMLR.

\bibitem[Kolobov et~al., 2012]{kolobov2012planning}
Kolobov, A. et~al. (2012).
\newblock {\em Planning with Markov decision processes: An AI perspective}, volume~17.
\newblock Morgan \& Claypool Publishers.

\bibitem[Lee et~al., 2022]{lee_2022}
Lee, S., Seo, Y., Lee, K., Abbeel, P., and Shin, J. (2022).
\newblock Offline-to-{Online} {Reinforcement} {Learning} via {Balanced} {Replay} and {Pessimistic} {Q}-{Ensemble}.
\newblock In {\em Proceedings of the 5th {Conference} on {Robot} {Learning}}, pages 1702--1712. PMLR.
\newblock ISSN: 2640-3498.

\bibitem[Li et~al., 2024]{li_settling_2024}
Li, G., Shi, L., Chen, Y., Chi, Y., and Wei, Y. (2024).
\newblock Settling the sample complexity of model-based offline reinforcement learning.
\newblock {\em The Annals of Statistics}, 52(1).

\bibitem[Li et~al., 2023]{li_2023}
Li, G., Zhan, W., Lee, J.~D., Chi, Y., and Chen, Y. (2023).
\newblock Reward-agnostic {Fine}-tuning: {Provable} {Statistical} {Benefits} of {Hybrid} {Reinforcement} {Learning}.
\newblock In Oh, A., Naumann, T., Globerson, A., Saenko, K., Hardt, M., and Levine, S., editors, {\em Advances in {Neural} {Information} {Processing} {Systems}}, volume~36, pages 55582--55615. Curran Associates, Inc.

\bibitem[Li et~al., 2025]{li_automatic}
Li, M., Zhang, J., and Bareinboim, E. (2025).
\newblock Automatic reward shaping from confounded offline data.
\newblock In {\em Forty-second International Conference on Machine Learning}.

\bibitem[Lykouris et~al., 2021]{lykouris2021corruption}
Lykouris, T., Simchowitz, M., Slivkins, A., and Sun, W. (2021).
\newblock Corruption-robust exploration in episodic reinforcement learning.
\newblock In {\em Conference on Learning Theory}, pages 3242--3245. PMLR.

\bibitem[Ma et~al., 2025]{ma2025highly}
Ma, H., Luo, Z., Vo, T.~V., Sima, K., and Leong, T.-Y. (2025).
\newblock Highly efficient self-adaptive reward shaping for reinforcement learning.
\newblock In {\em The Thirteenth International Conference on Learning Representations}.

\bibitem[Maurer and Pontil, 2009]{maurer_empirical_2009}
Maurer, A. and Pontil, M. (2009).
\newblock Empirical {Bernstein} {Bounds} and {Sample} {Variance} {Penalization}.

\bibitem[Nakamoto et~al., 2023]{nakamoto2023cal}
Nakamoto, M., Zhai, S., Singh, A., Sobol~Mark, M., Ma, Y., Finn, C., Kumar, A., and Levine, S. (2023).
\newblock Cal-ql: Calibrated offline rl pre-training for efficient online fine-tuning.
\newblock {\em Advances in Neural Information Processing Systems}, 36:62244--62269.

\bibitem[Ng et~al., 1999]{ng1999policy}
Ng, A.~Y., Harada, D., and Russell, S. (1999).
\newblock Policy invariance under reward transformations: Theory and application to reward shaping.
\newblock In {\em Icml}, volume~99, pages 278--287. Citeseer.

\bibitem[Pan and Yang, 2010]{PanYang2010}
Pan, S.~J. and Yang, Q. (2010).
\newblock A survey on transfer learning.
\newblock {\em IEEE Transactions on Knowledge and Data Engineering}, 22(10):1345--1359.

\bibitem[Pattanaik and Varshney, 2025]{pattanaik_2025}
Pattanaik, A. and Varshney, L.~R. (2025).
\newblock Online {Reinforcement} {Learning} with {Passive} {Memory}.
\newblock In {\em 2025 {American} {Control} {Conference} ({ACC})}, pages 3551--3557, Denver, CO, USA. IEEE.

\bibitem[Russo, 2019]{russo2019worst}
Russo, D. (2019).
\newblock Worst-case regret bounds for exploration via randomized value functions.
\newblock {\em Advances in neural information processing systems}, 32.

\bibitem[Schaal, 1996]{schaal1996learning}
Schaal, S. (1996).
\newblock Learning from demonstration.
\newblock {\em Advances in neural information processing systems}, 9.

\bibitem[Simchowitz and Jamieson, 2019]{simchowitz2019non}
Simchowitz, M. and Jamieson, K.~G. (2019).
\newblock Non-asymptotic gap-dependent regret bounds for tabular mdps.
\newblock {\em Advances in Neural Information Processing Systems}, 32.

\bibitem[Song et~al., 2022]{song2022hybrid}
Song, Y., Zhou, Y., Sekhari, A., Bagnell, J.~A., Krishnamurthy, A., and Sun, W. (2022).
\newblock Hybrid rl: Using both offline and online data can make rl efficient.
\newblock {\em arXiv preprint arXiv:2210.06718}.

\bibitem[Xie et~al., 2021]{xie_policy_2021}
Xie, T., Jiang, N., Wang, H., Xiong, C., and Bai, Y. (2021).
\newblock Policy {Finetuning}: {Bridging} {Sample}-{Efficient} {Offline} and {Online} {Reinforcement} {Learning}.
\newblock In {\em Advances in {Neural} {Information} {Processing} {Systems}}, volume~34, pages 27395--27407. Curran Associates, Inc.

\bibitem[Xu et~al., 2021]{xu_fine-grained_2021}
Xu, H., Ma, T., and Du, S. (2021).
\newblock Fine-{Grained} {Gap}-{Dependent} {Bounds} for {Tabular} {MDPs} via {Adaptive} {Multi}-{Step} {Bootstrap}.
\newblock In {\em Proceedings of {Thirty} {Fourth} {Conference} on {Learning} {Theory}}, pages 4438--4472. PMLR.
\newblock ISSN: 2640-3498.

\bibitem[Ye et~al., 2020]{ye2020towards}
Ye, D., Chen, G., Zhang, W., Chen, S., Yuan, B., Liu, B., Chen, J., Liu, Z., Qiu, F., Yu, H., et~al. (2020).
\newblock Towards playing full moba games with deep reinforcement learning.
\newblock {\em Advances in Neural Information Processing Systems}, 33:621--632.

\bibitem[Zanette and Brunskill, 2019]{zanette2019tighter}
Zanette, A. and Brunskill, E. (2019).
\newblock Tighter problem-dependent regret bounds in reinforcement learning without domain knowledge using value function bounds.
\newblock In {\em International Conference on Machine Learning}, pages 7304--7312. PMLR.

\bibitem[Zhai et~al., 2022]{zhai2022computational}
Zhai, Y., Baek, C., Zhou, Z., Jiao, J., and Ma, Y. (2022).
\newblock Computational benefits of intermediate rewards for goal-reaching policy learning.
\newblock {\em Journal of Artificial Intelligence Research}, 73:847--896.

\bibitem[Zhang and Wang, 2021]{zhang_provably_2021}
Zhang, C. and Wang, Z. (2021).
\newblock Provably efficient multi-task reinforcement learning with model transfer.
\newblock volume~34, pages 19771--19783. Curran Associates, Inc.

\bibitem[Zhang et~al., 2024]{zhang2024settling}
Zhang, Z., Chen, Y., Lee, J.~D., and Du, S.~S. (2024).
\newblock Settling the sample complexity of online reinforcement learning.
\newblock In {\em The Thirty Seventh Annual Conference on Learning Theory}, pages 5213--5219. PMLR.

\bibitem[Zheng et~al., 2025]{zheng2025learning}
Zheng, H., Shen, L., Tang, A., Luo, Y., Hu, H., Du, B., Wen, Y., and Tao, D. (2025).
\newblock Learning from models beyond fine-tuning.
\newblock {\em Nature Machine Intelligence}, 7(1):6--17.

\bibitem[Zhu et~al., 2019]{zhu2019dexterous}
Zhu, H., Gupta, A., Rajeswaran, A., Levine, S., and Kumar, V. (2019).
\newblock Dexterous manipulation with deep reinforcement learning: Efficient, general, and low-cost.
\newblock In {\em 2019 International Conference on Robotics and Automation (ICRA)}, pages 3651--3657. IEEE.

\bibitem[Zhuang et~al., 2020]{zhuang2020comprehensive}
Zhuang, F., Qi, Z., Duan, K., Xi, D., Zhu, Y., Zhu, H., Xiong, H., and He, Q. (2020).
\newblock A comprehensive survey on transfer learning.
\newblock {\em Proceedings of the IEEE}, 109(1):43--76.

\bibitem[Zou et~al., 2019]{zou2019reward}
Zou, H., Ren, T., Yan, D., Su, H., and Zhu, J. (2019).
\newblock Reward shaping via meta-learning.
\newblock {\em arXiv preprint arXiv:1901.09330}.

\end{thebibliography}

\appendix
\onecolumn

\part*{Appendix}                          
\addcontentsline{toc}{part}{Appendix}     
\etocsetnexttocdepth{section}             
\localtableofcontents                     

%
%

\section*{Notation}
\label{sec:notation}

This page provides a summary of the notation used throughout the appendix.

\begin{longtable}{@{}p{0.3\linewidth}p{0.65\linewidth}@{}}
\toprule
\textbf{Symbol} & \textbf{Description} \\
\midrule
\endfirsthead

\multicolumn{2}{c}%
{{\bfseries \tablename\ \thetable{} -- continued from previous page}} \\
\toprule
\textbf{Symbol} & \textbf{Description} \\
\midrule
\endhead

\bottomrule
\endfoot

\multicolumn{2}{l}{\textbf{General Mathematical Notation}} \\[2pt]
$[n]$ & The set $\{1, 2, \dots, n\}$. \\
$\mathbf{1}(\cdot)$ & The indicator function. \\
$\Delta(\mathcal{X})$ & The set of probability distributions over a set $\mathcal{X}$. \\
$\langle P, f \rangle$ & Expectation of a function $f$ w.r.t. a distribution $P$. Formally, $\langle P(\cdot|s,a), f \rangle \coloneqq \mathbb{E}_{s' \sim P(\cdot|s,a)}[f(s')]$. \\

\multicolumn{2}{l}{\textbf{Markov Decision Process (MDP) Components}} \\[2pt]
$\mathcal{S}, \mathcal{A}, H, \rho$ & State space, action space, horizon, and initial state distribution. \\
$\mathcal{S}_h$ & The set of states at step (layer) $h \in [H]$. \\
$P_h^{\star}(\cdot|s,a)$ & The true (unknown) transition probability distribution at step $h$. \\
$r_h(s,a)$ & The deterministic (known) reward function at step $h$. \\
$\pi$ & A policy, defined as a sequence of functions $\{\pi_h\}_{h=1}^H$. \\
$V_h^\pi, Q_h^\pi$ & The value and action-value functions for a policy $\pi$. \\
$V_h^\star, Q_h^\star$ & The optimal value and action-value functions. \\

\multicolumn{2}{l}{\textbf{Learning Phases and Data Structures}} \\[2pt]
$K$ & Total number of trajectories in the offline dataset. \\
$T$ & Total number of episodes in the online learning phase. \\
$\mathcal{D}$ & The offline dataset, containing $K$ trajectories. \\
$\mathcal{D}^{(h)}$ & The $h$-th disjoint subset of trajectories from $\mathcal{D}$. \\
$\tau^{(k)}$ & The $k$-th trajectory from the offline dataset $\mathcal{D}$. \\
$\tau_t$ & The $t$-th trajectory generated during the online phase. \\
$\mathcal{G}_K$ & The $\sigma$-algebra generated by the offline dataset $\mathcal{D}$. \\
$\mathcal{J}_t$ & The filtration generated by the first $t$ online trajectories. \\
$\mathcal{F}_t$ & The combined filtration $\sigma(\mathcal{G}_K \cup \mathcal{J}_t)$. \\

\multicolumn{2}{l}{\textbf{Empirical Models and Visitation Counts}} \\[2pt]
$N_h^t(s,a)$ & \textbf{Online} visitation count of $(s,a,h)$ before online episode $t$. \\
$N_h^{(h)}(s,a)$ & \textbf{Offline} visitation count of $(s,a,h)$ in the dataset split $\mathcal{D}^{(h)}$. \\
$\widehat{P}_h^t(\cdot|s,a)$ & \textbf{Online} empirical transition model before episode $t$. \\
$\widehat{P}_h^{(h)}(\cdot|s,a)$ & \textbf{Offline} empirical transition model estimated from $\mathcal{D}^{(h)}$. \\

\multicolumn{2}{l}{\textbf{Value Envelopes and Derived Quantities}} \\[2pt]
$\lowq_h, \highq_h$ & Generic lower and upper envelopes for $Q_h^\star$ (from Assumption 1). \\
$\lowv_h, \highv_h$ & Generic lower and upper envelopes for $V_h^\star$ (from Assumption 1). \\
$\olowq_h, \ohighq_h$ & Specific Q-envelopes produced by the offline algorithm (Algorithm 2). \\
$\olowv_h, \ohighv_h$ & Specific V-envelopes produced by the offline algorithm (Algorithm 2). \\
$D_h(s)$ & Width of the V-envelope at state $s$: $\highv_h(s) - \lowv_h(s)$. \\
$M_h(s)$ & Midpoint of the V-envelope at state $s$: $\frac{1}{2}(\highv_h(s) + \lowv_h(s))$. \\
$R_h$ & Range of the V-envelope over layer $h$: $\max_{s' \in \mathcal{S}_h} \highv_h(s') - \min_{s' \in \mathcal{S}_h} \lowv_h(s')$. \\
$D^{\max}, R^{\max}$ & Maximum width and range over all layers, i.e., $\max_{h \in [H]} D_h^{\max}$ and $\max_{h \in [H]} R_h$. \\

\multicolumn{2}{l}{\textbf{Algorithm-Specific Terms}} \\[2pt]
$\widehat{V}_h^t, \widehat{Q}_h^t$ & Optimistic value estimates computed at the start of online episode $t$. \\
$\bont(s,a)$ & Exploration bonus used in the \textbf{online} algorithm at episode $t$. \\
$\boff(s,a)$ & Confidence bonus used in the \textbf{offline} algorithm (Algorithm 2). \\
$\sigma_{h+1}^t(s,a)$ & Variance-like term used to compute the online bonus $\bont$. \\
$L, L_1, L_3$ & Logarithmic factors appearing in bonuses and regret bounds. \\
$\mathrm{PairEff}$ & Effective state-action pairs for Q-shaping: $\{(s,a) | \highq_{h(s)}(s,a) \ge V^\star_{h(s)}(s)\}$. \\
$\PS, \PPS, \BPS$ & Pseudo-suboptimality sets used in the analysis of V-shaping (see \cite{gupta_unpacking_2022}) \\

\multicolumn{2}{l}{\textbf{Empirical Statistics}} \\[2pt]
$\mathbf{Z} = (Z_1, \dots, Z_n)$ & Vector of i.i.d. random variables taking values in $[0,1]$. \\
$V_n(\mathbf{Z})$ & Unbiased sample variance : $\displaystyle V_n(\mathbf{Z}) = \frac{1}{n(n-1)}\sum_{1 \le i < j \le n}(Z_i - Z_j)^2$. \\[6pt]
$\widehat{\mathrm{Var}}(\mathbf{Z})$ & Biased sample variance: $\displaystyle \widehat{\mathrm{Var}}(\mathbf{Z}) = \frac{1}{n}\sum_{i=1}^n \Big(Z_i - \frac{1}{n}\sum_{i=1}^n Z_i\Big)^2$. \\
\bottomrule
\end{longtable}
\section{Background}

In this section, we recall the definition of a Markov decision process (subsection \ref{appsubsec:MDP}) and present online (subsection \ref{appsubsec:online}) and offline (subsection \ref{appsubsec:offline}) learning algorithms in this setting.

\subsection{Markov decision process (MDP)}\label{appsubsec:MDP}

\begin{definition}[Finite-Horizon MDP]
We define a finite-horizon, episodic, tabular Markov Decision Process (MDP) as the tuple $\mathcal{M}=(\mathcal{S}, \mathcal{A}, P^{\star}, r, H, \rho)$, where
\begin{itemize}
    \item $\mathcal{S}$ is the finite state space, and $\mathcal{A}$ is the finite action space. States are indexed by the step $h \in \{1,\dots,H\}$, so $\mathcal{S} = \cup_{h=1}^{H} \mathcal{S}_h$.
    \item $H \in \mathbb{N}$ is the horizon.
    \item $P_h^{\star}(\cdot|s,a)$ is the transition probability kernel from $(s,a) \in \mathcal{S}_h \times \mathcal{A}$ to $\mathcal{S}_{h+1}$.
    \item $r_h(s,a) \in [0,1]$ is the deterministic reward function at step $h$.
    \item $\rho$ is the initial state ditribution at step $h=1$.
    \item $V_h^\pi$ and $Q_h^\pi$ are the value and action-value functions for a policy $\pi$. The optimal value functions are $V_h^\star = \sup_\pi V_h^\pi$ and $Q_h^\star$, with $V_{H+1}^\star(s) \equiv 0$ for all $s$.
\end{itemize}
\end{definition}

\paragraph{Value functions.}
For any policy $\pi=\{\pi_h\}_{h=1}^H$, we define the (state) value function at step $h$ as the expected return from $(s,h)$:
\[
V_h^\pi(s)
\coloneqq
\mathbb{E}^\pi\!\left[\sum_{j=h}^{H} r_j(S_j,A_j)\,\middle|\, S_h=s\right],
\]
and the optimal value function as $V_h^\star(s)\coloneqq \sup_{\pi} V_h^\pi(s)$.

\paragraph{Sample space and $\sigma$-algebra.}
Define the offline and online episode product spaces
\[
\Omega_{\mathrm{off}}
=
\prod_{j=1}^{K}\Bigl(\mathcal S\times(\mathcal A\times[0,1]\times\mathcal S)^{H}\Bigr),
\qquad
\Omega_{\mathrm{on}}
=
\prod_{t=1}^{T}\Bigl(\mathcal S\times(\mathcal A\times[0,1]\times\mathcal S)^{H}\Bigr).
\]
The full sample space is the product
\[
\Omega=\Omega_{\mathrm{off}}\times\Omega_{\mathrm{on}}.
\]
An element \(\omega\in\Omega\) can be written as \(\omega=(\omega^{\mathrm{off}},\omega^{\mathrm{on}})\) where
\[
\omega^{\mathrm{off}}
=
\bigl(
(s_1^1,a_1^1,r_1^1,s_2^1,\dots,s_H^1,a_H^1,r_H^1,s_{H+1}^1),
\dots,
(s_1^K,a_1^K,r_1^K,s_2^K,\dots,s_{H+1}^K)
\bigr),
\]
\[
\omega^{\mathrm{on}}
=
\bigl(
(s_1^1,a_1^1,r_1^1,s_2^1,\dots,s_H^1,a_H^1,r_H^1,s_{H+1}^1),
\dots,
(s_1^T,a_1^T,r_1^T,s_2^T,\dots,s_{H+1}^T)
\bigr),
\]
with \(s_h^\cdot\in\mathcal S\), \(a_h^\cdot\in\mathcal A\), \(r_h^\cdot\in[0,1]\) for \(h\in\{1,\dots,H\}\), and \(s_{H+1}\) is a terminal state.
Let \(\Sigma_{\mathcal S}=2^{\mathcal S}\), \(\Sigma_{\mathcal A}=2^{\mathcal A}\) be the respective discrete $\sigma$-algebras and \(\mathcal{B}([0,1])\) be the Borel $\sigma$-algebra on \([0,1]\).
Equip \(\Omega_{\mathrm{off}}\) and \(\Omega_{\mathrm{on}}\) with the product $\sigma$-algebras
\[
\mathcal T_{\mathrm{off}}
=
\bigotimes_{j=1}^{K}\Bigl(\Sigma_{\mathcal S}\otimes(\Sigma_{\mathcal A}\otimes\mathcal{B}([0,1])\otimes\Sigma_{\mathcal S})^{\otimes H}\Bigr),
\]
\[
\mathcal T_{\mathrm{on}}
=
\bigotimes_{t=1}^{T}\Bigl(\Sigma_{\mathcal S}\otimes(\Sigma_{\mathcal A}\otimes\mathcal{B}([0,1])\otimes\Sigma_{\mathcal S})^{\otimes H}\Bigr).
\]
The $\sigma$-algebra on \(\Omega\) is the product
\[
\mathcal T=\mathcal T_{\mathrm{off}}\otimes\mathcal T_{\mathrm{on}}.
\]

The algorithm has an offline stage and an online stage, in that order:
\paragraph{Offline data generation:}
The offline dataset is made up of $K$ i.i.d.\ trajectories
$\mathcal{D}=\{\tau^{(i)}\}_{i=1}^K$, where
\[\tau^{(i)}=\big(S^{(i)}_1,A^{(i)}_1,R^{(i)}_1,\dots,S^{(i)}_H,A^{(i)}_H,R^{(i)}_H\big)\]
sampled from the MDP described above under an arbitrary behavior policy $\pi^{\mathsf b}=\{\pi^{\mathsf b}_h\}_{h=1}^H$. We suppose that our behavior policy is such that all states and stages have positive visitation probability. The MDP being tabular, this gives us a bound on visitation probabilities:
\[
\forall h \in [H], s \in \mathcal{S}_h, \quad \mathbb{P}^{\pi^{\mathsf{b}}}(s_h = s) \geq d^{\mathsf b}_{\min} > 0
\]
\paragraph{Online data generation:}
We will run an algorithm for $T$ episodes, each episode $t$ will produce a deterministic policy $\pi^{t+1} = \{\pi^{t+1}_h\}_{h=1}^H$ that will be used to select actions in episode $t+1$. By deterministic policy we mean a policy that always chooses the same action given the same $(s,h)$ pair, the policy itself is the ouput of an algorithm subject to randomness of the MDP. We note these \textbf{online} random trajectories: 
\[
\forall t \in [T], \quad \tau_t \coloneq \big( S_1^t, A_1^t, R_1^t, \dots, S_H^t, A_H^t, R_H^t\big).
\]
\begin{definition}[Episode Filtration]
\label{def:episode_filtration}
We denote the \textbf{offline} $\sigma$-algebra as 
\[
\mathcal{G}_K = \sigma (\mathcal{D})
\]
Let the first $t$ \textbf{online} trajectories be \( (\tau_1,\dots,\tau_t) \). Define
\[
\mathcal{F}_{t} = \sigma ( \mathcal{D} \cup \{ \tau_1, \dots, \tau_{t} \} ),
\]
with $\mathcal{F}_0 = \mathcal{G}_K$. Intuitively, this means that all the information available before acting in episode $t$ is exactly \( \mathcal{F}_{t-1} \).
\end{definition}

\paragraph{Notation}
The inner product notation $\langle \cdot, \cdot \rangle$ is used as a convenient shorthand for the expectation of a function with respect to a next-state distribution. For any function $f: \mathcal{S}_{h+1} \to \mathbb{R}$ and any transition kernel $P^{\star}_h(\cdot|s,a)$, we define:
\[
\langle P^{\star}_h(\cdot|s,a), f \rangle \coloneqq \mathbb{E}_{s' \sim P^{\star}_h(\cdot|s,a)} [f(s')] = \sum_{s' \in \mathcal{S}_{h+1}} P^{\star}_h(s'|s,a) f(s').
\]
For brevity, we may sometimes refer to the entire probability distribution $P^{\star}_h(\cdot|s,a)$ as $P^{\star}_{h,s,a}$. For example, the expected value of the optimal next-state value function $V_{h+1}^{\star}$ is written as $\langle P_h^{\star}(\cdot|s,a), V_{h+1}^{\star} \rangle = \langle P_{h,s,a}^{\star}, V_{h+1}^{\star} \rangle$.

\subsection{Online learning setting}\label{appsubsec:online}

Our main working condition in the \textbf{online} phase is the existence of these known bounds for the optimal value function. We will derive the following results for any $(\lowv_{h},\highv_{h}, \lowq_{h}, \highq_{h})$ and only replace them with $(\olowv_h, \ohighv_h, \olowq_{h}, \ohighq_{h})$ from Section~\ref{sec:offline} when applying our specific \textbf{offline} phase. We restate Assumption~\ref{assumption:sandwich} here for completeness:

There exists $\delta>0$ and a set of $\mathcal{G}_K$-measurable envelope functions 
\[
\crl{\lowv_{h}}_{h=1}^{H}, \quad \crl{\highv_{h}}_{h=1}^{H}, \quad \crl{\lowq_h}_{h=1}^{H}, \quad \crl{\highq_h}_{h=1}^{H}
\]
such that the event $\mathcal{E}^{\mathrm{Off}}_\delta$ defined by the following conditions holding for all $h \in [H]$ and $(s,a) \in \mathcal{S}_h \times \mathcal{A}$:
\begin{align}
 \lowq_h(s,a) &\leq Q_h^\star(s,a) \leq \highq_h(s,a) \label{eq:q_sandwich_appendix} \\
 \lowv_{h}(s) &\leq V_h^\star(s) \leq \highv_{h}(s) \label{eq:v_sandwich_appendix}
\end{align}
occurs with probability at least $1-\delta$. The $V$-envelopes $\crl{\lowv_{h}, \highv_{h}}_{h=1}^{H}$ are derived from the $Q$-envelopes $\crl{\lowq_{h}, \highq_{h}}_{h=1}^{H}$ as follows:
\[
\lowv_{h}(s) \coloneqq \max_{a \in \mathcal{A}} \lowq_h(s,a) \quad \text{and} \quad \highv_{h}(s) \coloneqq \max_{a \in \mathcal{A}} \highq_h(s,a).
\]
\begin{remark}
The $Q$-envelopes will only be used in the $Q$-shaping section of the proof. The $V$-shaping part functions solely on the $V$-envelopes.
\end{remark}

\noindent
Then, we define the \emph{width} and \emph{midpoint} of the \sandwich{} as:
\begin{equation}\label{eq:def_width_midpoint}
D_h(s) \coloneqq \highv_{h}(s) - \lowv_{h}(s) \ge 0, \quad M_h(s) \coloneqq \frac{1}{2} (\highv_{h}(s) + \lowv_{h}(s) ),
\end{equation}
as well as the quantities
\begin{equation}\label{eq:def_width_max}
R_{h+1}\coloneqq\max_{s'}\highv_{h+1}(s')-\min_{s'}\lowv_{h+1}(s')\quad\text{and}\quad D_{h}^{\max} \coloneqq \max_{s \in \mathcal{S}_h} D_h(s).
\end{equation}
The above quantifies are all measurable with respect to  $\mathcal{G}_K$. 

\subsection*{Counts and empirical models in the online phase}
\label{sec:counts}
For each episode index $t\in [T]$ and step $h\in[H]$, state $s\in\mathcal{S}_h$ and action $a\in\mathcal{A}$, define the empirical counts up to episode $t$ and subsequent empirical model:
\[
N_h^t(s,a)=\sum_{i=1}^{t-1}\mathbf{1}\{S_h^i=s,A_h^i=a\}.
\]
\[
N_h^t(s,a,s')=\sum_{i=1}^{t-1}\mathbf{1}\{S_h^i=s,A_h^i=a,S_{h+1}^i=s'\}.
\]
\[
\widehat{P}_h^{t}(\cdot\mid s,a)=
\begin{cases}
\dfrac{N_h^t(s,a,\cdot)}{N_h^t(s,a)} & \text{if } N_h^t(s,a)>0,\\[1em]
\dfrac{1}{|\mathcal{S}_{h+1}|} & \text{if } N_h^t(s,a)=0.
\end{cases}
\]

 Define the following exploration bonus for the online algorithm:  When $N_h^t(s,a) \in \lbrace0,1\rbrace$, we set $\bont(s,a) = R_{h+1}$, otherwise for $N_h^t(s,a) \geq 2$

\begin{equation}\label{eq:def_bonus_online}
\bont(s,a)\coloneqq \min \left \lbrace 
c_1 \sigma_{h+1}^t(s,a) \sqrt{\frac{ L}{N_h^t(s,a)}} + 
c_2 \frac{R_{h+1} L}{N_h^t(s,a)}, R_{h+1} \right \rbrace,
\end{equation}
with
\[
\sigma_{h+1}^t(s,a) \coloneqq \sqrt{\mathrm{Var}_{s'\sim\widehat{P}_h^{t}(\cdot|s,a)} (M_{h+1}(s') )}
 + \dfrac12\sqrt{\mathbb{E}_{s'\sim\widehat{P}_h^{t}(\cdot|s,a)} [D_{h+1}(s')^2 ]},
\]
and 
\[
 L\coloneqq\ln\left(\dfrac{8|\mathcal{S}| |\mathcal{A}| H T}{\delta} \right),\quad c_1 =2,\quad c_2=\frac{14}{3}
\]
The above quantities are all defined using information up to and including the $(t-1)$-th episode and the offline phase. Formally this means they are all measurable with respect to $\mathcal{F}_{t-1}$.

Following \cite{gupta_unpacking_2022}, define 
\[
Q_h^{\highv}(s,a)=r_h(s,a)+\langle P_h^\star( \cdot\mid s,a),  \highv_{h+1} \rangle
\]
and the following sets:
\begin{align}\label{eq:def_PS_Delta}
\PS
&=\Big\{(s,a,h): Q_h^{\highv}(s,a)\le V_h^\star(s)-\Delta\Big\},\\
\PPS 
&=  \left \lbrace s \in \mathcal{S} ~\vert~ \text{all feasible paths from } \rho \text{ to } s \text{ intersect } \PS \right \rbrace\\
\BPS
&= \left \lbrace (s,a,h) \in \PS ~ \vert ~ s \in \PPS  \right \rbrace
\end{align}
These sets are used in the analysis of the V-shaping part of the regret.

\subsection{Offline Learning Setting}\label{appsubsec:offline}
\label{sec:offline}

\paragraph{Offline data and $H$-way trajectory split:}
We have a batch dataset \(\mathcal{D}\) consisting of \(K\) i.i.d. trajectories collected by a behavior policy \(\pi^{\mathsf{b}}=\{\pi^{\mathsf{b}}_h\}_{h=1}^H\) from an initial distribution \(\rho^{\mathsf{b}}\).

Following \cite{xie_policy_2021}, we split the \(K\) trajectories into \(H\) disjoint subsets \(\{\mathcal{D}^{(h)}\}_{h=1}^{H}\); write \(m_h=|\mathcal{D}^{(h)}| \ge \left \lfloor K/H \right \rfloor\) as evenly as possible and note that \(\sum_h m_h=K\). For each step \(h\), define per-step counts

\begin{equation}\label{eq:off_step_count}
 N_h^{(h)}(s,a) = \left|\{k \in [K], (s_h^k,a_h^k)=(s,a) \in \mathcal{D}^{(h)}\}\right| \quad\text{and}\quad
N_h^{(h)}(s,a,s') = \left|\{k \in [K], (s_h^k,a_h^k,s_{h+1}^k)=(s,a,s') \in \mathcal{D}^{(h)}\}\right|,
\end{equation}
and next-state samples
\begin{equation}\label{eq:off_emp_proba}
\widehat{P}_h^{(h)}(s' \mid s,a) = \frac{N_h^{(h)}(s,a,s')}{N_h^{(h)}(s,a)},
\end{equation}
using only \(\mathcal{D}^{(h)}\).
This $H$-split guarantees the per-step independence needed for the bonuses.
All value functions that appear inside the bonuses at step \(h\) are computed from data that excludes \(\mathcal{D}^{(h)}\). This technique is suboptimal with respect to sample complexity and has been replaced with other procedures (see \cite{li_settling_2024}) but this is not the scope of this article.

We define the following bonus for the offline algorithm:
When $N_h^{(h)}(s,a) \in \lbrace 0,1\rbrace$, we set $\boff(s,a) = H-h$. Otherwise for $N_h^{(h)}(s,a) \geq 2$:
\begin{equation}
\label{eq:bonus_h_split_biased_notation}
\boff(s,a)
\coloneqq
\min \left \lbrace c_1 \sqrt{\frac{\max\left\{\widehat{\mathrm{Var}}_{\widehat{P}^{(h)}_h(\cdot|s,a)}[\ohighv_{h+1}],
\widehat{\mathrm{Var}}_{\widehat{P}^{(h)}_h(\cdot|s,a)}[\olowv_{h+1}]\right\}L_1}{N_h^{(h)}(s,a)}}
+
c_2\frac{(H-h)L_1}{N_h^{(h)}(s,a)},
H-h
\right \rbrace
\end{equation}
where
\[ L_1 \coloneq \log \left(\dfrac{8|\mathcal{S}||\mathcal{A}|H}{\delta}\right), \quad c_1 = 2, \quad c_2 = \frac{14}{3} .\]

Here, for a function $f : \mathcal{S} \rightarrow \mathbb{R}$, $\widehat{\mathrm{Var}}$ is the \textbf{biased} sample variance:
\[
\widehat{\mathrm{Var}}_{\widehat{P}^{(h)}_h(\cdot|s,a)}[f] \coloneqq \frac{1}{n}\sum_{i=1}^n \left(f(S'_i) - \frac{1}{n}\sum_{i=1}^n f(S'_i)\right)^2.
\]
All the above quantities are measurable with respect to $\mathcal{G}_K$.
\begin{lemma}[Offline Bonus Validity] 
\label{lem:off_bonus_validity} 
With probability $\ge 1 - \delta$, for every $h, s, a$, the bonuses provide valid confidence intervals: 
\begin{equation}\label{eq:off_bonus_validity} 
|(P^{\star}-\widehat{P}^{(h)})\olowv_{h+1} |\le \boff(s,a) \quad\text{and}\quad |(P^{\star}-\widehat{P}^{(h)})\ohighv_{h+1}| \le \boff(s,a). 
\end{equation} 
\end{lemma}
We define then the high probability event 
\begin{equation}\label{eq:def_Eoff}
\mathcal{E}^{\mathrm{Off}}_\delta=\{(s,a,h) \in \mathcal{S} \times \mathcal{A} \times [H]:\eqref{eq:off_bonus_validity} \text{ holds}\}.
\end{equation}
\begin{proof}
Fix a triplet $(s,a,h)$. The bonus $\boff(s,a)$ and the empirical transition $\widehat{P}_h^{(h)}(\cdot\mid s,a)$ are computed using only data from the trajectories in $\mathcal{D}^{(h)}$. The value functions $\ohighv_{h+1}$ and $\olowv_{h+1}$ are recursively computed backwards from step $H$ to $h+1$. These steps use $\{\widehat{P}_j^{(j)}, \boffp{j}\}_{j=h+1}^H$ that are constructed exclusively from the trajectories in $\{\mathcal{D}^{(j)}\}_{j=h+1}^H$. Since $\{\mathcal{D}^{(j)}\}_{j=1}^H$ are disjoint by construction, the data in $\mathcal{D}^{(h)}$ is statistically independent of all data used to construct $\ohighv_{h+1}$ and $\olowv_{h+1}$.
Let $n = N_h^{(h)}(s,a)$ and let $\{S'_i\}_{i=1}^n$ be the $n$ i.i.d. next states observed after $(s,a)$ at step $h$, drawn from the true transition function $P^\star(\cdot|s,a)$. Let $\delta'\in(0,1)$.
Since $0 \le \ohighv_{h+1}(s') \le H-h$, we can define normalized random variables $Z_i = \ohighv_{h+1}(S'_i)/(H-h)$, which take values in $[0,1]$ and apply Theorem~\ref{thm:empirical_bernstein} stating that for the fixed triplet, with probability at least $1 - \delta'$:

\begin{align*}
|\langle P_{h,s,a}^\star - \widehat{P}_{h,s,a}^{(h)},\ohighv_{h+1}\rangle| & = (H-h)\left|\mathbb{E}[Z] - \dfrac{\sum_{i}^{n} Z_i}{n}\right| \\
& \le (H-h) \left( 2\sqrt{\frac{\widehat{\mathrm{Var}}_{\widehat{P}^{(h)}_h(\cdot|s,a)}[\ohighv_{h+1}] \ln(4/\delta')}{(H-h)^2 n}} + \frac{14 \ln(4/\delta')}{3n} \right) \\
& = 2\sqrt{\frac{\widehat{\mathrm{Var}}_{\widehat{P}^{(h)}_h(\cdot|s,a)}[\ohighv_{h+1}] \ln(4/\delta')}{n}} + \frac{14(H-h)\ln(4/\delta')}{3n}.
\end{align*}
The same argument applied to $\olowv_{h+1}$ yields 
\begin{align*}
|\langle P_{h,s,a}^\star - \widehat{P}_{h,s,a}^{(h)}, \olowv_{h+1} \rangle| 
& \leq \sqrt{\frac{2 \widehat{\mathrm{Var}}_{\widehat{P}^{(h)}_h(\cdot|s,a)}[\olowv_{h+1}] \ln(4/\delta')}{n}} + \frac{14(H-h)\ln(4/\delta')}{3n}.
\end{align*}
To ensure this holds simultaneously for all $(s,a,h)$ and for both $\ohighv_{h+1}$ and $\olowv_{h+1}$, we take a union bound over all $2|\mathcal{S}||\mathcal{A}|H$ possibilities setting $\delta' = \delta/2|\mathcal{S}||\mathcal{A}|H$.
\end{proof}

\begin{lemma}[Offline Ordering]
On the event $\mathcal E^{\mathrm{Off}}_{\delta}$, for all $(s,a,h)$,
\[
\underline V_h(s)\le V_h^\star(s)\le \overline V_h(s)
\quad\text{and}\quad
\underline Q_h(s,a)\le Q_h^\star(s,a)\le \overline Q_h(s,a).
\]
\end{lemma}

\begin{proof}
By construction,
\[
\overline Q_h(s,a)=r_h(s,a)+\langle \widehat{P}^{(h)}_{h,s,a},\overline V_{h+1}\rangle+b_h^{\mathrm{off}}(s,a),
\quad
\underline Q_h(s,a)=r_h(s,a)+\langle \widehat{P}^{(h)}_{h,s,a},\underline V_{h+1}\rangle-b_h^{\mathrm{off}}(s,a),
\]
and $\overline V_h(s)=\max_{a}\overline Q_h(s,a)$, $\underline V_h(s)=\max_{a}\underline Q_h(s,a)$ with $\overline V_{H+1}=\underline V_{H+1}=0$.  
On $\mathcal E^{\mathrm{Off}}_{\delta}$ the confidence bounds hold for all $(s,a,h)$:
\[
\big|\langle P^{\star}_{h,s,a}-\widehat{P}^{(h)}_{h,s,a},\overline V_{h+1}\rangle\big|\le b_h^{\mathrm{off}}(s,a),
\quad
\big|\langle P^{\star}_{h,s,a}-\widehat{P}^{(h)}_{h,s,a},\underline V_{h+1}\rangle\big|\le b_h^{\mathrm{off}}(s,a).
\]

The proof is done by backward induction on $h$: For $h=H+1$ the claim holds. Assume $\underline V_{h+1}\le V_{h+1}^\star\le \overline V_{h+1}$. Then, for any $(s,a)$,
\[
Q_h^\star(s,a)=r_h(s,a)+\langle P^{\star}_{h,s,a},V_{h+1}^\star\rangle
\le r_h(s,a)+\langle P^{\star}_{h,s,a},\overline V_{h+1}\rangle.
\]
Decomposing and using the confidence bound,
\[
\langle P^{\star}_{h,s,a},\overline V_{h+1}\rangle
=\langle \widehat{P}^{(h)}_{h,s,a},\overline V_{h+1}\rangle+\langle P^{\star}_{h,s,a}-\widehat{P}^{(h)}_{h,s,a},\overline V_{h+1}\rangle
\le \langle \widehat{P}^{(h)}_{h,s,a},\overline V_{h+1}\rangle+b_h^{\mathrm{off}}(s,a),
\]
so $Q_h^\star(s,a)\le \overline Q_h(s,a)$. Similarly, since $V_{h+1}^\star\ge \underline V_{h+1}$,
\[
Q_h^\star(s,a)\ge r_h(s,a)+\langle P^{\star}_{h,s,a},\underline V_{h+1}\rangle
\ge r_h(s,a)+\langle \widehat{P}^{(h)}_{h,s,a},\underline V_{h+1}\rangle-b_h^{\mathrm{off}}(s,a)
=\underline Q_h(s,a).
\]
Taking the maximum over $a$ yields
\[
\underline V_h(s)=\max_{a}\underline Q_h(s,a)\le \max_{a}Q_h^\star(s,a)=V_h^\star(s)
\le \max_{a}\overline Q_h(s,a)=\overline V_h(s).
\]
This completes the proof.
\end{proof}

\begin{lemma}
  \label{lem:high_prob_Nmin}
  Suppose $\left\lfloor K/H\right\rfloor \ge \dfrac{8}{d^{\mathsf b}_{\min}}\log\left(\dfrac{H|\mathcal{S}||\mathcal{A}|}{\delta}\right)$. With probability at least $1-\delta$, we have for all $(s,h) \in \mathcal{S} \times [H]$,
\begin{equation}\label{eq:bound_Nh}
N_h^{(h)}(s, \overline{\pi}_h(s)) \ge \frac{1}{2} \left \lfloor \frac{K}{H} \right \rfloor d^{\mathsf{b}}_h(s,\overline{\pi}_h(s))\ge \frac{1}{2} \left \lfloor \frac{K}{H} \right \rfloor d^{\mathsf{b}}_{\min} =: N_{\min},
\end{equation}
where
 \[d^{\mathsf{b}}_{\min} = \min_{\substack{h,s,a \\ d^{\mathsf{b}}_h(s,a) > 0}} d^{\mathsf{b}}_h(s,a).\]
\end{lemma}
We define then the high probability event 
\begin{equation}\label{eq:def_Emin}
\mathcal{E}^{\mathrm{min}}_\delta=\{(s,h) \in \mathcal{S} \times [H]:\eqref{eq:bound_Nh} \text{ holds}\}.
\end{equation}
\begin{proof}
  For any triplet $(h,s,a)$, the sample count $N_h^{(h)}(s,a)$ is computed using the dataset $\mathcal{D}^{(h)}$, which contains $m_h = |\mathcal{D}^{(h)}| \ge \lfloor K/H \rfloor$ trajectories. Therefore, $N_h^{(h)}(s,a)$ is a sum of $m_h$ independent and identically distributed Bernoulli random variables, each with parameter $d^{\mathsf{b}}_{h}(s,a)$. The expectation of this count is $\mathbb{E}[N_h^{(h)}(s,a)] = m_h d^{\mathsf{b}}_{h}(s,a)$.

We use a multiplicative version of the Chernoff bound for sums of Bernoulli variables. For any $\epsilon \in (0,1)$:
\[
\mathbb{P}(N_h^{(h)}(s, a) < (1 - \epsilon)m_h d^{\mathsf{b}}_{h}(s,a)) < \exp\left(-\frac{\epsilon^2 m_h d^{\mathsf{b}}_{h}(s,a)}{2}\right).
\]
Since $m_h \ge \lfloor K/H \rfloor$ and $d^{\mathsf{b}}_{h}(s,a) \ge d^{\mathsf{b}}_{\min}$ for state-action pairs covered by the behavior policy, we get:
\[
\mathbb{P}(N_h^{(h)}(s, a) < (1 - \epsilon)m_h d^{\mathsf{b}}_{h}(s,a)) < \exp\left(-\frac{\epsilon^2\lfloor K/H \rfloor d^{\mathsf{b}}_{\min}}{2}\right).
\]
To ensure this bound holds simultaneously for all state-action pairs that could be chosen by the policy $\overline{\pi}$, we take a union bound over all possible triplets $(h,s,a) \in [H] \times \mathcal{S} \times \mathcal{A}$. Let $\mathcal{E}_{h,s,a}$ be the event that $N_h^{(h)}(s, a) < (1 - \epsilon)m_h d^{\mathsf{b}}_{h}(s,a)$. The probability of at least one such event occurring is:
\begin{align*}
    \mathbb{P}\left(\bigcup_{h,s,a} \mathcal{E}_{h,s,a}\right) & \le \sum_{h=1}^H \sum_{s\in\mathcal{S}} \sum_{a \in \mathcal{A}} \mathbb{P}(\mathcal{E}_{h,s,a}) \\
    & \le \sum_{h,s,a} \exp\left(-\frac{\epsilon^2 m_h d^{\mathsf{b}}_{h}(s,a)}{2}\right) \\
    & \le \sum_{h,s,a} \exp\left(-\frac{\epsilon^2 \lfloor K/H \rfloor d^{\mathsf{b}}_{\min}}{2}\right) \\
    & \le H S A \exp\left(-\frac{\epsilon^2 \lfloor K/H \rfloor d^{\mathsf{b}}_{\min}}{2}\right).
\end{align*}
We want this total failure probability to be less than a desired $\delta$. We set the right-hand side to be less than or equal to $\delta$:
\[
    H S A\exp\left(-\frac{\epsilon^2 \lfloor K/H \rfloor d^{\mathsf{b}}_{\min}}{2}\right) \le \delta \quad \implies \quad \lfloor K/H \rfloor \ge \frac{2}{\epsilon^2d^{\mathsf{b}}_{\min}}\log\left(\frac{H|\mathcal{S}||\mathcal{A}|}{\delta}\right).
\]
By setting the deviation parameter $\epsilon = 1/2$, we get the condition:
\[
    \lfloor K/H \rfloor \ge \frac{8}{d^{\mathsf{b}}_{\min}}\log\left(\frac{H|\mathcal{S}||\mathcal{A}|}{\delta}\right).
\]
If this condition on $K$ is met, then with probability at least $1-\delta$, for all $(h,s,a)$, we have $N_h^{(h)}(s, a) \ge \dfrac{1}{2} m_h d^{\mathsf{b}}_{h}(s,a)$. 

Since this holds for all actions $a$, it must also hold for the specific action $a = \overline{\pi}_h(s)$ chosen by the policy. Therefore, for all $(s,h) \in \mathcal{S}\times[H]$:
\begin{equation*}
    N_h^{(h)}(s, \overline{\pi}_h(s)) \ge \dfrac{1}{2} m_h d^{\mathsf{b}}_h(s,\overline{\pi}_h(s)) \ge \frac{1}{2} \left \lfloor \frac{K}{H} \right \rfloor d^{\mathsf{b}}_h(s,\overline{\pi}_h(s)) \ge \frac{1}{2} \left \lfloor \frac{K}{H} \right \rfloor d^{\mathsf{b}}_{\min}.
\end{equation*}
\end{proof}

\begin{proposition}
\label{prop:offline_width}
On the event $\mathcal{E}^{\mathrm{Off}}_\delta$ (given by \eqref{eq:def_Eoff}), for all $(s,h)\in\mathcal{S}\times[H]$,
\begin{equation}
\label{eq:width_known_theo}
\ohighv_h(s)-\olowv_h(s)
\le
2\sum_{k=h}^{H}
\mathbb{E}^{\overline{\pi},\widehat{P}}\left[
b_k\left(S_k,\overline{\pi}_k(S_k)\right)\middle| S_h=s
\right].
\end{equation}
Furthermore on
$\mathcal{E}^{\mathrm{Off}}_\delta \cap \mathcal{E}^{\mathrm{min}}_\delta$ (where $\mathcal{E}^{\mathrm{min}}_\delta$ is given by \eqref{eq:def_Emin}),
we get: 
\begin{equation}
\label{eq:width_known_real}
\ohighv_1(s) - \olowv_1(s)
\le
2H^2\left[
c_1\sqrt{\frac{2HL_1}{Kd^{\mathsf{b}}_{\min}}}+c_2\frac{2HL_1}{K d^{\mathsf{b}}_{\min}}
\right].
\end{equation}
\end{proposition}
\begin{proof}
Fix $(s,h)$ and define the width $w_t\coloneqq \ohighv_t-\olowv_t$ for $t\in\{h,\dots,H+1\}$, so $w_{H+1}\equiv 0$. 
Then
\begin{align*}
w_h(s)
&= \max_{a}\ohighq_h(s,a)-\max_{a}\olowq_h(s,a)
\le \ohighq_h(s,\ovpihs)-\olowq_h(s,\ovpihs)\\
&= \left(r_h+\widehat{P}_h\ohighv_{h+1}+\boff\right)(s,\ovpihs)
 - \left(r_h+\widehat{P}_h\olowv_{h+1}-\boff\right)(s,\ovpihs)\\
&= \widehat{P}_h\left(\ohighv_{h+1}-\olowv_{h+1}\right)(s,\ovpihs) + 2\boff\left(s,\ovpihs\right)\\
&= \mathbb{E}_{S_{h+1}\sim \widehat{P}_h(\cdot\mid s,\ovpihs)}\left[w_{h+1}(S_{h+1})\right]
 + 2\boff\left(s,\overline{\pi}_h(s)\right).
\end{align*}
Define $T_t\coloneqq 2b_t\left(\cdot,\overline{\pi}_t(\cdot)\right)$ for $t\in\{h,\dots,H\}$. Rewriting the above inequality gives a one-step recursion:
\begin{equation}
\label{eq:one_step_recursion}
w_h(s)\le T_h(s)+\mathbb{E}_{S_{h+1}\sim \widehat{P}_h(\cdot\mid s,\overline{\pi}_h(s))}\left[w_{h+1}(S_{h+1})\right].
\end{equation}
Take conditional expectations along the trajectories using $(\overline{\pi},\widehat{P})$:
\[
\overline{w}_t(s)\coloneqq \mathbb{E}^{\overline{\pi},\widehat{P}}\left[w_t(S_t)\mid S_h=s\right],\quad
\overline{T}_t(s)\coloneqq \mathbb{E}^{\overline{\pi},\widehat{P}}\left[T_t(S_t)\mid S_h=s\right].
\]
Applying the recursion \eqref{eq:one_step_recursion} yields
$\overline{w}_h(s)\le \overline{T}_h(s)+\overline{w}_{h+1}(s)$, i.e., $\overline{w}_h(s)-\overline{w}_{h+1}(s)\le \overline{T}_h(s)$.
Summing from $t=h$ to $H$ produces the telescoping sum
\[
\overline{w}_h(s)\le\sum_{t=h}^{H}\overline{T}_t(s)
=2\sum_{t=h}^{H}\mathbb{E}^{\overline{\pi},\widehat{P}}\left[
\boffp{t}\left(S_t,\overline{\pi}_t(S_t)\right)\middle| S_h=s
\right].
\]
Since $w_h$ is deterministic given $s$, $\overline{w}_h(s)=w_h(s)$, giving us \eqref{eq:width_known_theo}. Now applying Lemma~\ref{lem:high_prob_Nmin} to \eqref{eq:width_known_theo} we get that on the event $\mathcal{E}^{\mathrm{min}}_\delta$, for every $(t,s')\in[H]\times\mathcal{S}$ we have
\[
N_t^{(t)}\big(s',\overline{\pi}_t(s')\big)\ge N_{\min}=\frac{1}{2}\Big\lfloor\frac{K}{H}\Big\rfloor d^{\mathsf{b}}_{\min}.
\]
Moreover, since $0\le \ohighv_{t+1}(s'),\olowv_{t+1}(s')\le H-t$ for all $s'$, the biased empirical variance satisfies
\[
\widehat{\mathrm{Var}}_{\widehat{P}^{(t)}_t(\cdot\mid s',\overline{\pi}_t(s'))}\big[\ohighv_{t+1}\big]\le \dfrac{(H-t)^2}{4},
\quad
\widehat{\mathrm{Var}}_{\widehat{P}^{(t)}_t(\cdot\mid s',\overline{\pi}_t(s'))}\big[\olowv_{t+1}\big]\le \dfrac{(H-t)^2}{4}.
\]
Plugging these two bounds in the definition of $\boffp{t}$ and using $N_t^{(t)}\big(s',\overline{\pi}_t(s')\big)\ge N_{\min}$ yields, for all $s'$,
\[
\boffp{t}\big(s',\overline{\pi}_t(s')\big)
\le
c_1(H-t)\sqrt{\frac{L_1}{N_{\min}}}+c_2\frac{(H-t)L_1}{N_{\min}}.
\]
Taking expectations in \eqref{eq:width_known_theo} and summing over $t=h,\dots,H$ gives
\[
\ohighv_h(s)-\olowv_h(s)
\le
2\sum_{t=h}^{H}\mathbb{E}^{\overline{\pi},\widehat{P}}\left[\boffp{t}\big(S_t,\overline{\pi}_t(S_t)\big)\mid S_h=s\right]
\le
2\sum_{t=h}^{H}\left[c_1(H-t)\sqrt{\frac{L_1}{N_{\min}}}+c_2\frac{(H-t)L_1}{N_{\min}}\right].
\]
Since $\sum_{t=h}^{H}(H-t)=\sum_{u=0}^{H-h}u=(H-h)(H-h+1)/2$, we arrive at the explicit bound
\begin{equation}
\label{eq:width_known_real}
\ohighv_h(s)-\olowv_h(s)
\le
2(H-h)(H-h+1)\left[
c_1\sqrt{\frac{L_1}{N_{\min}}}+c_2\frac{L_1}{N_{\min}}
\right],
\end{equation}
which holds simultaneously for all $(s,h)\in\mathcal{S}\times[H]$ on the event $\mathcal{E}^{\mathrm{Off}}_\delta\cap\mathcal{E}^{\mathrm{min}}_\delta$, provided $N_{\min}\ge 2$. Replacing $N_{\min}$ by its definition gives the fully explicit form
\[
\ohighv_h(s)-\olowv_h(s)
\le
2(H-h)(H-h+1)\left[
c_1\sqrt{\frac{2L_1}{\left\lfloor\tfrac{K}{H}\right\rfloor d^{\mathsf{b}}_{\min}}}+c_2\frac{2L_1}{\left\lfloor\tfrac{K}{H}\right\rfloor d^{\mathsf{b}}_{\min}}
\right].
\]
Assume \(K\) is divisible by \(H\), so \(\lfloor K/H \rfloor = K/H\). Using \(H(H-1) \le H^2\), the bound at \(h=1\) becomes
\[
\ohighv_1(s) - \olowv_1(s)
\le
2H^2\left[
c_1\sqrt{\frac{2HL_1}{Kd^{\mathsf{b}}_{\min}}}+c_2\frac{2HL_1}{K d^{\mathsf{b}}_{\min}}
\right].
\]
\end{proof}

\section{Preliminaries: Conditional Bernstein}
\label{sec:cond_bernstein}
The following theorem provides a rigourous proof of an empirical Bernstein inequality tailored to our setting, this is necessary as we need to control deviations using data-dependent (random) bounds in our analysis.

\begin{theorem}[Conditionally Independent Empirical Bernstein]
\label{thm:cond_empirical_bernstein_final}
Let $(\Omega,\mathcal{F},\mathbb{P})$ be a probability space and let $\mathcal{G}\subseteq\mathcal{F}$ be a sub-$\sigma$-algebra.
Let $(\mathcal{X},\mathcal{B})$ be a measurable space and let $X_1,\dots,X_n$ ($n\ge 2$) be i.i.d.\ $\mathcal{X}$-valued random variables such that $\sigma(X_1,\dots,X_n)$ is independent of $\mathcal{G}$.

Let $g:\Omega\times\mathcal{X}\to\mathbb{R}$ be $(\mathcal{G}\otimes\mathcal{B})$-measurable.
Let $\underline{g},\overline{g}:\Omega\to\mathbb{R}$ be $\mathcal{G}$-measurable, assume $\underline{g}\le \overline{g}$ almost surely, and define the (random) range length
\[
R:=\overline{g}-\underline{g}.
\]
Let $E\in\mathcal{G}$ be an event such that on $E$,
\[
\underline{g}(\omega)\le g(\omega,x)\le \overline{g}(\omega)
\qquad
\text{for all }x\in\mathcal{X}.
\]

For each $i=1,\dots,n$, define the random variable $Z_i:\Omega\to\mathbb{R}$ by
\[
Z_i(\omega):=g(\omega,X_i(\omega)).
\]
Let
\[
\overline{Z}(\omega):=\frac{1}{n}\sum_{i=1}^n Z_i(\omega),
\qquad
\widehat{\mathrm{Var}}(\mathbf Z)(\omega)
:=\frac{1}{n}\sum_{i=1}^n (Z_i(\omega)-\overline{Z}(\omega))^2,
\qquad
\mu_g(\omega):=\mathbb E[Z_1\mid \mathcal{G}](\omega).
\]
Then for any $\delta\in(0,1)$,
\[
\mathbb{P}\Bigl(
E\cap\Bigl\{
\bigl|\mu_g-\overline{Z}\bigr|
\le
2\sqrt{\frac{\widehat{\mathrm{Var}}(\mathbf Z)\ln(4/\delta)}{n}}
+
\frac{14\,R\,\ln(4/\delta)}{3n}
\Bigr\}
\Bigr)
\ge
(1-\delta)\,\mathbb{P}(E),
\]
equivalently,
\[
\mathbb{P}\Bigl(
\bigl\{
\bigl|\mu_g-\overline{Z}\bigr|
>
2\sqrt{\frac{\widehat{\mathrm{Var}}(\mathbf Z)\ln(4/\delta)}{n}}
+
\frac{14\,R\,\ln(4/\delta)}{3n}
\bigr\}\cap E
\Bigr)
\le
\delta\,\mathbb{P}(E).
\]
\end{theorem}

\begin{proof}
Fix $n\ge 2$ and $\delta\in(0,1)$. Define the bad event $\mathcal A \in \mathcal{F}$ by
\[
\mathcal A
:=
\left\{
\omega \in \Omega :
\left|\mu_g(\omega)-\overline{Z}(\omega)\right|
>
2\sqrt{\frac{\widehat{\mathrm{Var}}(\mathbf Z)(\omega)\ln(4/\delta)}{n}}
+\frac{14R(\omega)\ln(4/\delta)}{3n}
\right\}.
\]
We show $\mathbb{P}(\mathcal A\cap E)\le \delta\,\mathbb{P}(E)$.

Since $E\in\mathcal{G}$, by the tower property of conditional expectation,
\[
\mathbb{P}(\mathcal A\cap E)
=
\mathbb E[\mathbf 1_{\mathcal A}\mathbf 1_E]
=
\mathbb E\left[\mathbb E[\mathbf 1_{\mathcal A}\mid\mathcal{G}]\mathbf 1_E\right]
=
\mathbb E\left[\mathbb{P}(\mathcal A\mid\mathcal{G})\mathbf 1_E\right].
\]

Let $\mathbb{P}(\cdot \mid \mathcal{G})(\omega_0)$ denote the regular conditional probability measure on $(\Omega, \mathcal{F})$ evaluated at a fixed outcome $\omega_0$. Because it is a regular conditional probability, it is a valid probability measure for $\mathbb{P}$-almost every $\omega_0 \in \Omega$. Therefore, it suffices to prove that for $\mathbb{P}$-almost every $\omega_0 \in E$,
\[
\mathbb{P}(\mathcal A\mid\mathcal{G})(\omega_0)\le \delta.
\]

Fix such an $\omega_0\in E$ where the regular conditional probability is well-defined. Under the probability measure $\mathbb{P}(\cdot\mid\mathcal{G})(\omega_0)$, any $\mathcal{G}$-measurable random variable $H$ is almost surely equal to the constant $H(\omega_0)$. Consequently, $\underline{g}, \overline{g}$, $R$, and $\mu_g$ act as the deterministic constants $\underline{g}(\omega_0), \overline{g}(\omega_0), R(\omega_0)$, and $\mu_g(\omega_0)$.

Define the deterministic function $g_{\omega_0}:\mathcal{X}\to\mathbb{R}$ by $g_{\omega_0}(x):=g(\omega_0,x)$. 
Because $\sigma(X_1,\dots,X_n)$ is independent of $\mathcal{G}$, the joint law of $(X_1,\dots,X_n)$ under the conditional probability $\mathbb{P}(\cdot\mid\mathcal{G})(\omega_0)$ is identical to its unconditional law; in particular, $(X_i)_{i=1}^n$ remain i.i.d under $\mathbb{P}(\cdot\mid\mathcal{G})(\omega_0)$.

Since the $\mathcal{G}$-measurable first coordinate of $Z_i(\cdot) = g(\cdot, X_i(\cdot))$ is almost surely $\omega_0$ under this measure, we have the $\mathbb{P}(\cdot\mid\mathcal{G})(\omega_0)$-almost sure identity:
\[
Z_i(\omega) = g(\omega_0, X_i(\omega)) = g_{\omega_0}(X_i(\omega)).
\]
Consequently, we have that $(Z_i)_{i=1}^n = (g_{\omega_0}(X_i))_{i=1}^n$ almost surely under $\mathbb{P}(\cdot\mid\mathcal{G})(\omega_0)$ and so $(Z_i)_{i=1}^n$ are i.i.d. under $\mathbb{P}(\cdot\mid\mathcal{G})(\omega_0)$.

Also, since $\omega_0\in E$, we have the deterministic bounds
\[
\underline{g}(\omega_0)\le g_{\omega_0}(x)\le \overline{g}(\omega_0)\qquad\text{for all }x \in \mathcal{X},
\]
which implies $Z_i\in[\underline{g}(\omega_0),\overline{g}(\omega_0)]$ almost surely under $\mathbb{P}(\cdot\mid\mathcal{G})(\omega_0)$.

If $R(\omega_0)=0$, then $\underline{g}(\omega_0) = \overline{g}(\omega_0)$, meaning $Z_1=\cdots=Z_n=\mu_g(\omega_0)=\overline{Z}$ almost surely under $\mathbb{P}(\cdot\mid\mathcal{G})(\omega_0)$. In this case, the strict inequality defining $\mathcal A$ is impossible, yielding $\mathbb{P}(\mathcal A\mid\mathcal{G})(\omega_0) = 0 \le \delta$.

Assume now $R(\omega_0)>0$ and define, under the probability measure $\mathbb{P}(\cdot\mid\mathcal{G})(\omega_0)$, the rescaled random variables
\[
Y_i:=\frac{Z_i-\underline{g}(\omega_0)}{R(\omega_0)}\in[0,1],
\qquad
\overline Y:=\frac{1}{n}\sum_{i=1}^n Y_i,
\qquad
\widehat{\mathrm{Var}}(\mathbf Y):=\frac{1}{n}\sum_{i=1}^n (Y_i-\overline Y)^2.
\]
Then $(Y_i)_{i=1}^n$ are i.i.d.\ in $[0,1]$ under $\mathbb{P}(\cdot\mid\mathcal{G})(\omega_0)$, and
\[
\mathbb E_{\mathbb{P}(\cdot\mid\mathcal{G})(\omega_0)}[Y_1]
=
\frac{\mathbb E_{\mathbb{P}(\cdot\mid\mathcal{G})(\omega_0)}[Z_1]-\underline{g}(\omega_0)}{R(\omega_0)}
=
\frac{\mu_g(\omega_0)-\underline{g}(\omega_0)}{R(\omega_0)}.
\]
Furthermore, it holds algebraically that
\[
\mu_g(\omega_0)-\overline{Z}
=
R(\omega_0)\Bigl(\mathbb E_{\mathbb{P}(\cdot\mid\mathcal{G})(\omega_0)}[Y_1]-\overline Y\Bigr),
\qquad
\widehat{\mathrm{Var}}(\mathbf Z)
=
R(\omega_0)^2\,\widehat{\mathrm{Var}}(\mathbf Y).
\]

Applying Theorem~\ref{thm:empirical_bernstein} to $(Y_i)_{i=1}^n$ with parameter $\delta$, under the proper probability measure $\mathbb{P}(\cdot\mid\mathcal{G})(\omega_0)$, we obtain
\[
\mathbb{P}\left(
\left|
\mathbb E_{\mathbb{P}(\cdot\mid\mathcal{G})(\omega_0)}[Y_1]-\overline Y
\right|
>
2\sqrt{\frac{\widehat{\mathrm{Var}}(\mathbf Y)\ln(4/\delta)}{n}}
+
\frac{14\ln(4/\delta)}{3n}
\ \middle|\ \mathcal{G}
\right)(\omega_0)
\le
\delta.
\]
Multiplying the inequality inside the absolute value by the constant $R(\omega_0)$, the event in the probability precisely reconstructs the definition of $\mathcal A$ evaluated at $\omega_0$.
Therefore, for $\mathbb{P}$-almost every $\omega_0\in E$,
\[
\mathbb{P}(\mathcal A\mid\mathcal{G})(\omega_0)\le \delta.
\]

Finally, substituting this back into the unconditional expectation,
\[
\mathbb{P}(\mathcal A\cap E)
=
\mathbb E\left[\mathbb{P}(\mathcal A\mid\mathcal{G})\mathbf 1_E\right]
\le
\mathbb E[\delta\,\mathbf 1_E]
=
\delta\,\mathbb{P}(E),
\]
which is the desired conclusion.
\end{proof}

\section{V-shaping: Proof of Theorem \ref{thm:v_shaping_regret}}

\subsection{Preliminary Lemmas}

\begin{lemma}[Online Bonus Validity]
\label{lem:on_bonus_validity}
Define :
\begin{equation}\label{eq:def_Eon}
\mathcal{E}^{\mathrm{On}}_\delta=\{(t,s,a,h) \in [T]\times\mathcal{S} \times \mathcal{A} \times [H]:\eqref{eq:on_bonus_validity}\,\mathrm{holds}\}.
\end{equation}
where for every $t, h, s, a$, \eqref{eq:on_bonus_validity} is
\begin{equation}\label{eq:on_bonus_validity}
| \left \langle P^{\star}_{h,s,a}-\widehat{P}_{h,s,a}^{t},V_{h+1}^{\star} \right \rangle| \le \bont(s,a) \quad \text{and} \quad | \left \langle P_{h,s,a}^{\star}-\widehat{P}_{h,s,a}^{t}, \highv_{h+1} \right \rangle| \le \bont(s,a).
\end{equation}
Then we have that
\[ \mathbb{P}(\mathcal{E}^{\mathrm{Off}}_\delta \cap \mathcal{E}^{\mathrm{On}}_\delta) \ge 1 - 2\delta.\]
\end{lemma}

\begin{proof}
Define $(X_{k,h}^{s,a})_{(s,a)\in\mathcal S\times\mathcal A,\ h\in[H],\ k\in[T]}$ taking values in $\mathcal S$ and independent across $s,a,h,k$, such as
  \[
  \mathbb{P}(X_{k,h}^{s,a}=s')=P_h(s'\mid s,a)\ \text{ for all }s'\in\mathcal S.
  \]
Our assumption on the model is that the offline dataset is independent of the array $(X_{k,h}^{s,a})$, and that the online dynamics are \emph{coupled} to the pre-generated transitions by the following protocol:
\[
\textit{Whenever $(s,a)$ is visited at step $h$ for the $k$-th time during the online phase, then the next state is } X_{k,h}^{s,a}.
\]
For $(s,a,h)$ and $t\in[T]$, let
\[
\mathcal T_h^t(s,a)=\{t'\le t:\ S_h^{t'}=s,\ A_h^{t'}=a\},\qquad
N_h^t(s,a)=|\mathcal T_h^t(s,a)|.
\]
Order $\mathcal T_h^t(s,a)$ as $t'_1<\dots<t'_{N_h^t(s,a)}$ and write $S_i'=S_{h+1}^{t'_i}$. By the coupling,
\[
S_i'=X_{i,h}^{s,a}\qquad\text{for }i=1,\dots,N_h^t(s,a).
\]
The empirical kernel based on the first $n$ i.i.d. transitions at step $h$ is
\[
\widehat P_h^{[n]}(s'\mid s,a)=\frac1n\sum_{k=1}^n\mathbf 1_{\{X_{k,h}^{s,a}=s'\}},
\]
and the online empirical kernel is
\[
\widehat P_h^{t}(s'\mid s,a)=\frac{1}{N_h^t(s,a)}\sum_{t'\in\mathcal T_h^t(s,a)}\mathbf 1_{\{S_{h+1}^{t'}=s'\}}.
\]
By the coupling, for $n=N_h^t(s,a)\ge1$,
\[
\widehat P_h^{t}(\cdot\mid s,a)=\widehat P_h^{[n]}(\cdot\mid s,a).
\]

Fix $(h,s,a)$ and $n \in [T]$ with $n \ge 2$. We will apply  Theorem~\ref{thm:cond_empirical_bernstein_final} to two distinct choices for the target function $g: \Omega \times \mathcal{S} \to \mathbb{R}$:
\begin{enumerate}
    \item $g(\omega, s') = V_{h+1}^\star(s')$ (which is deterministic and thus $\mathcal{G}_K$-measurable).
    \item $g(\omega, s') = \highv_{h+1}(\omega, s')$ (which is a random variable computed from the offline dataset).
\end{enumerate}
We use $\mathcal{G} = \mathcal{G}_K$ and the event $E = \mathcal{E}^{\mathrm{Off}}_\delta$.
We use the coupled next-state variables $X_{1,h}^{s,a}, \dots, X_{n,h}^{s,a}$ as the sample $X_1, \dots, X_n$.
We define $\underline{g} \coloneqq \min_{z} \lowv_{h+1}(z)$ and $\overline{g} \coloneqq \max_{z} \highv_{h+1}(z)$, yielding range $R_{h+1} = \overline{g} - \underline{g}$.
We define $Z_i \coloneqq g(\omega, X_{i,h}^{s,a}())$. This establishes the direct parallel to the theorem's result terms:
    \begin{align*}
        \overline{Z}(\omega) &= \frac{1}{n}\sum_{i=1}^n g(\omega, X_{i,h}^{s,a}(\omega)) = \langle \widehat{P}_h^{[n]}(\cdot\mid s,a)(\omega), g(\omega,\cdot) \rangle, \\
        \mu_g(\omega) &= \langle P_{h,s,a}^\star(\cdot), g(\omega,\cdot) \rangle, \\
        \widehat{\mathrm{Var}}(\mathbf{Z}) &= \frac{1}{n}\sum_{i=1}^n (Z_i - \overline{Z})^2 = \widehat{\mathrm{Var}}_{\widehat{P}_h^{[n]}(\cdot\mid s,a)}[g]
    \end{align*}

Applying Theorem~\ref{thm:cond_empirical_bernstein_final} with confidence level $\delta'$, we substitute these equivalences directly into the concentration bound. We obtain that for each choice of $g$:
\[
\mathbb{P}\left(
\mathcal E_\delta^{\mathrm{Off}} \cap \left\{
\left| \langle P_{h,s,a}^\star-\widehat P_h^{[n]}(\cdot\mid s,a), g \rangle \right|
\le
2\sqrt{\frac{\widehat{\mathrm{Var}}_{\widehat P_h^{[n]}(\cdot\mid s,a)}[g]\ln(4/\delta')}{n}}
+
\frac{14 R_{h+1}\ln(4/\delta')}{3n}
\right\}
\right)
\ge
(1-\delta')\mathbb{P}(\mathcal E_\delta^{\mathrm{Off}}).
\]
Furthermore, on the event $\mathcal E_\delta^{\mathrm{Off}}$, the values of $g$ are confined to an interval of length $R_{h+1}$, which implies the deterministic bound $| \langle P_{h,s,a}^\star-\widehat P_h^{[n]}, g \rangle | \le R_{h+1}$.

The next step is to bound the biased variance $\widehat{\mathrm{Var}}_{\widehat{P}^{t}_h(\cdot|s,a)}[g]$.
Recall 
\[ M_h = \frac{1}{2}(\highv_{h+1} + \lowv_{h+1}), \quad D_h = \highv_{h+1} - \lowv_{h+1}, \]
the fact that $\lowv_{h+1} \le g \le \highv_{h+1}$ implies that $|g - M_{h+1}| \le \dfrac{D_{h+1}}{2}$. Therefore, decomposing 
\[ g = M_{h+1} + (g - M_{h+1}) \]
and applying the Minkowski inequality to variances we get
\begin{align*}
\sqrt{\widehat{\mathrm{Var}}_{\widehat{P}^{t}_h(\cdot|s,a)}[g]} 
    & \leq 
\sqrt{\widehat{\mathrm{Var}}_{\widehat{P}^{t}_h(\cdot|s,a)}[M_{h+1}]} + \sqrt{\widehat{\mathrm{Var}}_{\widehat{P}^{t}_h(\cdot|s,a)}[g - M_{h+1}]} \\
    & \leq \sqrt{\widehat{\mathrm{Var}}_{\widehat{P}^{t}_h(\cdot|s,a)}[M_{h+1}]} + \sqrt{\widehat{\mathbb{E}}_{\widehat{P}^{t}_h(\cdot|s,a)}[(g - M_{h+1})^2]} \\
    & \leq 
\sqrt{\widehat{\mathrm{Var}}_{\widehat{P}^{t}_h(\cdot|s,a)}[M_{h+1}]} + \frac{1}{2}\sqrt{\widehat{\mathbb{E}}_{\widehat{P}^{t}_h(\cdot|s,a)}[D_{h+1}^2]} \\
    & = \sigma_{h+1}^t(s,a).
\end{align*}

Combining with the previous step and taking the minimum with $R_{h+1}$ yields exactly the online bonus form $\bont(s,a)$ used in \eqref{eq:on_bonus_validity}.

Choose $\delta' = \delta/(2|\mathcal S||\mathcal A|HT)$ and take a union bound over all $h \in [H]$, $(s,a)\in\mathcal S\times\mathcal A$, $n \in [T]$, and both choices $g \in \{V_{h+1}^\star,\highv_{h+1}\}$.
Since Theorem \ref{thm:cond_empirical_bernstein_final} controls each deviation event on $\mathcal E_\delta^{\mathrm{Off}}$ with conditional failure probability at most $\delta'$, we obtain
\[
\mathbb{P}\bigl(\mathcal E_\delta^{\mathrm{Off}} \cap \mathcal E_\delta^{\mathrm{On}}\bigr)
\ge
(1-\delta)\mathbb{P}(\mathcal E_\delta^{\mathrm{Off}}).
\]
Finally, Lemma \ref{lem:off_bonus_validity} gives $\mathbb{P}(\mathcal E_\delta^{\mathrm{Off}})\ge 1-\delta$, hence
\[
\mathbb{P}\bigl(\mathcal E_\delta^{\mathrm{Off}} \cap \mathcal E_\delta^{\mathrm{On}}\bigr)
\ge
(1-\delta)^2
\ge
1-2\delta.
\]
\end{proof}

\begin{lemma}[Total Ordering]
\label{lemma:total_ordering}
On the event $\mathcal{E}^{\mathrm{Off}}_\delta\cap \mathcal{E}^{\mathrm{On}}_\delta$, for all $(t,h,s,a)$,
\[
\lowv_{h}(s) \le V_h^\star(s) \le \widehat{V}_h^t(s) \le \highv_{h}(s),
\quad
Q_h^\star(s,a) \le \widehat{Q}_h^t(s,a)
 \quad \text{and } \quad
\widehat{Q}_h^t(s,a) \le Q_h^{\highv}(s,a)+2b_h^{t}(s,a).
\]
\end{lemma}

\begin{proof}
    Suppose $\mathcal{E}^{\mathrm{Off}}_\delta\cap \mathcal{E}^{\mathrm{On}}_\delta$, by backwards induction, with base case $V_{H+1}^\star \le \widehat{V}_{H+1}^t = 0$, supposing that for a fixed $h+1$ we have that 
    \[
    \forall s \in \mathcal{S} \quad V_{h+1}^\star(s) \le \widehat{V}_{h+1}^t(s), 
    \]
    then 
    \begin{align*}
        \widehat{Q}^t_h(s,a) - Q^{\star}_h(s,a) & = r_h(s,a) + \bont(s,a) + \langle \widehat{P}^t_h, \widehat{V}^t_{h+1} \rangle - r_h(s,a) - \langle P^{\star}_h, V^{\star}_{h+1} \rangle \\
        & \geq  \bont(s,a) + \langle \widehat{P}^t_h, V^{\star}_{h+1}\rangle - \langle P^{\star}_h, V^{\star}_{h+1} \rangle \\
        & \geq \bont(s,a) +\langle \widehat{P}^t_h - P^{\star}_h, V^{\star}_{h+1}\rangle \\
        & \geq 0,
    \end{align*}
    where the first inequality is the induction hypothesis and the last is by Lemma~\ref{lem:on_bonus_validity}. Taking the $\argmax$ over actions gives us 
    \[ V^{\star}_h(s) \leq \max_a \widehat{Q}^t_h(s,a).\]
    Furthermore as we are on $\mathcal{E}^{\mathrm{Off}}_\delta\cap \mathcal{E}^{\mathrm{On}}_\delta$ we have 
    \[ V^{\star}_h(s) \leq \highv_{h}(s),\]
    and so 
    \[ V^{\star}_h(s) \leq  \min \left \lbrace \highv_{h}(s),\max_a  
    \widehat{Q}^t_h(s,a)\right \rbrace = \widehat{V}^t_h(s)\]
    Rolling out the induction gives us for all $(s,a,h,t)$:
    \[
    V_h^\star(s) \le \widehat{V}_h^t(s) \quad \text{and} \quad Q_h^\star(s,a) \le \widehat{Q}_h^t(s,a)
    \]
Let us prove: 
\[
\widehat{Q}_h^t(s,a) \le Q_h^{\highv}(s,a)+2 \bont(s,a).
\] 
Furthermore, by definition of $\widehat{Q}_h^t(s,a)$ and $\widehat{V}^t_{h+1}$,
\begin{align*}
    \widehat{Q}_h^t(s,a) & = r_h(s,a) + \bont(s,a) + \langle \widehat{P}^t_h , \widehat{V}^t_{h+1} \rangle \\
    & \leq  r_h(s,a) + \bont(s,a) + \langle \widehat{P}^t_h , \highv_{h+1} \rangle.
\end{align*}
Using 
\[\langle \widehat{P}^t_h , \highv_{h+1} \rangle = \langle P^{\star}_h , \highv_{h+1} \rangle + \langle \widehat{P}^t_h - P^{\star}_h , \highv_{h+1}, \rangle\] and Lemma~\ref{lem:on_bonus_validity} and the definition of $Q_h^{\highv}$
\begin{align*}
    \widehat{Q}_h^t(s,a) &= r_h(s,a) + \bont(s,a) + \langle P^{\star}_h , \highv_{h+1} \rangle + \langle \widehat{P}^t_h - P^{\star}_h , \highv_{h+1} \rangle \\
    &\leq r_h(s,a) + 2 \bont(s,a) + \langle P^{\star}_h , \highv_{h+1} \rangle \\
    &= 2 \bont(s,a) + Q_h^{\highv}(s,a),
\end{align*}
which concludes the proof.
\end{proof}

\begin{proposition}
The bonus term \( \bont(s,a) \) satisfies
\begin{align}
\bont(s,a) \le 
c_1\left(\dfrac12R_{h+1} + \dfrac12 D^{\max}_{h+1}\right)\sqrt{\frac{L}{N_h^t(s,a)}}
 + 
c_2\frac{R_{h+1}L}{N_h^t(s,a)}.
\label{eq:bonus_bound}
\end{align}
\end{proposition}
\begin{proof}
Using the fact 
\[\forall s \in \mathcal{S}, \quad\min_{s'} \lowv_{h+1}(s') \leq M_{h+1}(s) \leq \max_{s'} \highv_{h+1}(s'),\]
we get by Popoviciu's inequality on $M_{h+1}$ that
\[
\sqrt{\mathrm{Var}(M_{h+1})}\le \dfrac12 R_{h+1}.
\]
Furthermore,
$$\dfrac12\sqrt{\mathbb{E}[D_{h+1}^2]}\le \dfrac12 D^{\max}_{h+1} \leq R_{h+1},$$ 
resulting in  
$$\sigma_{h+1}^t(s,a)\le R_{h+1}.$$
\end{proof}

\begin{proposition}
\label{prop:max_vis}
Define the pairwise threshold
\begin{equation}\label{eq:def_pairwise_threshold}
n_\Delta(h) \coloneqq \min \bigg\{n\in\mathbb{N}: c_1R_{h+1}\sqrt{\frac{L}{n}}+c_2 R_{h+1}\frac{L}{n} \le \frac{\Delta}{4} \bigg\}.
\end{equation}
Then, on $\mathcal{E}^{\mathrm{Off}}_\delta\cap \mathcal{E}^{\mathrm{On}}_\delta$, for every $(s,a,h)\in\PS$,
\begin{align}
    N_h^T(s,a) \le n_\Delta(h) \le 64 c_1^2 (1+R_{h+1})^2 \dfrac{L}{\Delta^2} 
    \label{eq:total_visitation_pseudosub}
\end{align}
\end{proposition}
\begin{proof}
From Lemma~\ref{lemma:total_ordering} we know that on the event $\mathcal{E}^{\mathrm{Off}}_\delta\cap \mathcal{E}^{\mathrm{On}}_\delta$, we have 
\[
\widehat{Q}_h^t(s,a) \le Q_h^{\highv}(s,a)+2 \bont(s,a).
\]
Take $(s,a,h)\in\PS$ such that  $Q_h^\highv(s,a) \le V_h^\star(s) - \Delta$. Combining these gives us 
\[
\widehat{Q}_h^t(s,a) \le V_h^\star(s) - \Delta + 2 \bont(s,a).
\]
This means that whenever $2\bont(s,a) \le \Delta/2$, we have 
\[
    \widehat{Q}_h^t(s,a) \le V_h^\star(s) - \dfrac{\Delta}{2} < V_h^\star(s). 
\]
Combining this with the optimism property of our estimate $V_h^\star(s) = Q_h^\star(s, \pi_h^\star(s)) \le \widehat{Q}_h^t(s,\pi_h^\star(s))$, we obtain 
\[
    \widehat{Q}_h^t(s,a) < \widehat{Q}_h^t(s,\pi_h^\star(s)),
\]
which implies that action $a$ will not be chosen in $(s,h)$, since our exploration is greedy. To show that $(s,a,h)$ is never visited again, it suffices to prove that once $b_h^t$ falls below a certain threshold, it never increases back. For this we bound $b_h^t$ by the strictly non-increasing function from \eqref{eq:bonus_bound}, further bounding $D^{\max}_h$ with $R_h$ we get: 
\[
b_h^{t}(s,a) \le 
c_1R_{h+1} \sqrt{\frac{L}{N_h^t(s,a)}}
 + 
c_2\frac{R_{h+1}L}{N_h^t(s,a)},
\]
where
\[
R_{h+1}\coloneqq\max_{s'}\highv_{h+1}(s')-\min_{s'}\lowv_{h+1}(s').
\]
By definition \eqref{eq:def_pairwise_threshold} of $n_\Delta(h)$, we have that 
\[ 
 N_h^T(s,a) \le n_\Delta(h).
\]
We now need to find a condition on $n$ such as the condition on \eqref{eq:def_pairwise_threshold} is verified. A sufficient condition for realizing this is:

\[c_1R_{h+1}\sqrt{\frac{L}{n}} \leq \dfrac{\Delta}{8} \quad \text{and} \quad c_2 R_{h+1}\frac{L}{n} \le \frac{\Delta}{8}\]
The first term gives us
\[n \geq 64 c_1^2 R_{h+1}^2 \dfrac{L}{\Delta^2},\]
and the second 
\[n \geq 8 c_2 R_{h+1} \dfrac{L}{\Delta}.\]
We would however like that joint condition on $n$ depending only on $\Delta^2$ and not $\Delta$. For this, we use the fact that any pertinent gap $\Delta$ is bounded by $1 + R_{h+1}$ at step $h$. This is because if $(s,a,h)\in\PS$ (defined in \eqref{eq:def_PS_Delta}) then
\begin{align*}
    \Delta & \leq V^\star_h(s) - Q^{\highv}_h(s,a) \\
            & \leq \max_{a'} \left ( r_h(s,a') + \langle P^\star_h, V^\star_{h+1} \rangle \right ) - r_h(s,a) - \langle P^\star_h, \highv_{h+1} \rangle \\
            & \leq 1 + \max_{s'} \highv_{h+1}(s') - \min_{s'}\lowv_{h+1}(s') \\
            & \leq 1 + R_{h+1}
\end{align*}
Where the third inequality is because $(\max_{a'} r_h(s,a) - r_h(s,a)) \leq 1$ and $\lowv \leq V^\star \leq \highv$.

This entails that 
\[ 1\leq \dfrac{1 + R_{h+1}}{\Delta} \quad \text{and so} \quad \dfrac{R_{h+1}}{\Delta} \leq \dfrac{1 + R_{h+1}}{\Delta} \leq \dfrac{(1 + R_{h+1})^2}{\Delta^2}.\]
This means that when 
\[n \geq 8c_2(1+R_{h+1})^2\dfrac{L}{\Delta^2} \quad \text{we have} \quad n \geq 8c_2R_{h+1}\dfrac{L}{\Delta}.\]
Similarly if
\[
n \geq 64 c_1^2 (1+R_{h+1})^2 \dfrac{L}{\Delta^2} \quad \text{then} \quad n \geq 64 c_1^2 R_{h+1}^2 \dfrac{L}{\Delta^2}.
\]
Comparing $64c_{1}^2$ and $8c_2$ ($c_1 = 2$ and $c_2 = 14/3$), we get the sufficient condition: 
\[n \geq 64 c_1^2 (1+R_{h+1})^2 \dfrac{L}{\Delta^2},\]
and so 
\[
n_{\Delta}(h) \leq 64 c_1^2 (1+R_{h+1})^2 \dfrac{L}{\Delta^2}. 
\]
\end{proof}
\noindent
\begin{proposition}[Regret Decomposition]
\label{prop:regret_decomp}
Note the event 
\[\mathcal{E}^{\mathrm{Conc}}_\delta \coloneq \left \lbrace \forall (t,s) \in [T] \times \mathcal{S} \mid \eqref{eq:decomp} ~\mathrm{holds}\right \rbrace \]
where \eqref{eq:decomp} is
\begin{equation}
\label{eq:decomp}
V^\star_{1}(s)-V^{\pi^t}_{1}(s)
 \le e\cdot
\sum_{h=1}^{H}\mathbb{E}^{\pi^t} \left[2b_h^t(s_h,a_h) + \nu_h^t(s_h,a_h) \Big| a_h \backsim \pi^t_h, s_{1}=s\right]
\end{equation}
and
\[
\nu_h^t(s,a) \coloneq D_{h+1}^{\max} d^{s,a}
    \frac{3 H L_3}
    {N_{h(s)}^t(s,a)} \qquad \text{and} \qquad L_3 \coloneq \ln\left( \dfrac{T|\mathcal{S}||\mathcal{A}|H}{\delta} \right).
\]
Then we have 
\[ \mathbb{P} \left (\mathcal{E}^{\mathrm{Off}}_\delta\cap \mathcal{E}^{\mathrm{On}}_\delta \cap \mathcal{E}^{\mathrm{Conc}}_\delta \right ) > 1- 3\delta.\]
\end{proposition}
\begin{proof}
Define
\[
\Delta_h(s') := \widehat{V}_h^{t}(s')-V_h^{\pi^t}(s'),\quad
a_h:=\pi_h^t(s_h).
\]
By greediness of \(\pi^t\) and the Bellman equation for \(V^{\pi^t}\),
\[
\widehat{V}_h^{t}(s_h)=r_h(s_h,a_h)+\langle \widehat{P}_{h,s_h,a_h}^{t},\widehat{V}_{h+1}^{t}\rangle+b_h^{t}(s_h,a_h)\quad\text{and}\quad V_h^{\pi^t}(s_h)=r_h(s_h,a_h)+\langle P_{h,s_h,a_h}^{\star},V_{h+1}^{\pi^t}\rangle.
\]
This gives us
\begin{align*}
\Delta_h(s_h) & = \left\langle\widehat{P}_{h,s_h,a_h}^{t},\widehat{V}_{h+1}^{t}\right\rangle- \left\langle P_{h,s_h,a_h}^{\star},V_{h+1}^{\pi^t}\right\rangle +\bont(s_h,a_h)\\
& = \left\langle \widehat{P}_{h,s_h,a_h}^{t} - P_{h,s_h,a_h}^{\star},\widehat{V}_{h+1}^{t} \right\rangle + \left\langle P_{h}^{\star},\widehat{V}_{h+1}^{t}-V_{h+1}^{\pi^t}\right\rangle + \bont(s_h,a_h)\\
& = \underbrace{\big\langle \widehat{P}_{h,s_h,a_h}^{t} - P_{h,s_h,a_h}^{\star},V_{h+1}^\star\big\rangle}_{\le \bont(s_h,a_h)}
+ \underbrace{\big\langle \widehat{P}_{h,s_h,a_h}^{t} - P_{h,s_h,a_h}^{\star},\widehat{V}_{h+1}^{t}-V_{h+1}^\star\big\rangle}_{\xi_h^{t}(s_h,a_h)}
+ \underbrace{\big\langle P_{h}^{\star},\widehat{V}_{h+1}^{t}-V_{h+1}^{\pi^t}\big\rangle}_{= \mathbb{E}[\Delta_{h+1}(s_{h+1})\mid s_h,a_h]}
+ \bont(s_h,a_h),
\end{align*}
so that, by Lemma \ref{lem:on_bonus_validity}, 
\begin{equation}
\Delta_h(s_h) \le 2\bont(s_h,a_h) + \xi_h^{t}(s_h,a_h) + \mathbb{E} \left[\Delta_{h+1}(s_{h+1})\mid s_h,a_h\right].
\label{eq:first_recur}
\end{equation}

From Lemma~\ref{lemma:total_ordering} we have 
\[ 
\forall s' \in \mathcal{S}, \quad \widehat{V}_{h+1}^t(s') - V_{h+1}^\star(s') \le \highv_{h+1}(s') - \lowv_{h+1}(s') \le D_{h+1}^{\max}.
\]
Thus we can set $B = D_{h+1}^{\max}$ in Lemma~\ref{lemma:borrowed} with $f = \widehat{V}_{h+1}^t - V_{h+1}^\star$. Then conditionally on $\mathcal{E}^{\mathrm{Off}}_\delta\cap \mathcal{E}^{\mathrm{On}}_\delta$, with probability $\geq 1 - \delta$ we have
\[
\xi_h^{t}(s,a) = \left \langle \widehat{P}_{h,s,a}^{t} -P_{h,s,a}^{\star} ,  \widehat{V}_{h+1}^t - V_{h+1}^\star \right \rangle \le \frac{\mathbb{E}_{s' \sim P_{h,s,a}^{\star}} \left[ \widehat{V}_h^t(s') - V_{h+1}^\star(s') \right]}{H} + \nu_h^{t}(s,a)
\]
where 
\[
\nu_h^{t}(s,a) = D_{h+1}^{\max} d^{s,a}
    \frac{3 H L_3}
    {N_{h(s)}^t(s,a)}
\]
This gives us 
\[
\mathbb{E}^{\pi^t} \left[\xi_h^{t}(s_h,a_h)\Big| s_{1}=s\right]
 \le \frac{1}{H}
\mathbb{E}^{\pi^t} \left[\widehat{V}_{h+1}^{t}(s_{h+1})-V_{h+1}^\star(s_{h+1})\Big| s_{1}=s\right] + 
\mathbb{E}^{\pi^t} \left[\nu_h^{t}(s_h,a_h)\Big| s_{1}=s\right].
\]
Since \(V^\star\ge V^{\pi^t}\) pointwise, \(\widehat{V}_{h+1}^{t}-V_{h+1}^\star\le \widehat{V}_{h+1}^{t}-V_{h+1}^{\pi^t}=\Delta_{h+1}\).
Define
\[
\overline{\Delta}_h(s):=\mathbb{E}^{\pi^t} \left[\Delta_h(s_h)\mid s_{1}=s\right],\quad
\overline{\nu}_h(s):=\mathbb{E}^{\pi^t} \left[\nu_h^{t}(s_h,a_h)\mid s_{1}=s\right].
\]
Taking \(\mathbb{E}^{\pi^t}[\cdot| s_{1}=s]\) in \eqref{eq:first_recur} yields, for all \(h\ge 1\),
\begin{equation}
\label{eq:recur2}
\overline{\Delta}_h(s) \le \mathbb{E}^{\pi^t} \left[2\bont(s_h,a_h)+\nu_h^{t}(s_h,a_h)\mid s_{1}=s\right] + \Big(1+\frac{1}{H}\Big)\overline{\Delta}_{h+1}(s),
\quad \overline{\Delta}_{H+1}(s)=0.
\end{equation}
Let $w_h:=( 1+1/H )^{-(H-h)}$ so that we have the recursion
\[w_{h+1}=\left ( 1+\dfrac{1}{H} \right ) w_h, \quad w_{1}=\left ( 1+\dfrac{1}{H} \right )^{-(H-1)}, \quad 0<w_h\le 1.\]
Then,
\[
\sum_{h=1}^H w_h\overline{\Delta}_h(s) \le \sum_{h=1}^H w_h\mathbb{E}^{\pi^t} \left[2b_h^{t}+\nu_h^{t}\mid s_{1}=s\right] + \sum_{h=1}^H w_{h+1}\overline{\Delta}_{h+1}(s),
\]
which gives
\[
w_{1}\overline{\Delta}_{1}(s) \le \sum_{h=1}^H w_h\mathbb{E}^{\pi^t} \left[2\bont+\nu_h^{t}\mid s_{1}=s\right]
 \le \sum_{h=1}^H \mathbb{E}^{\pi^t} \left[2\bont+\nu_h^{t}\mid s_{1}=s\right],
\]
because \(w_h\le 1\). Dividing by $w_{1}=(1+1/H)^{-(H-1)}$ gives
\begin{align*}
\overline{\Delta}_{1}(s) & \le \left ( 1+\dfrac{1}{H} \right )^{H-1} \sum_{h=1}^H \mathbb{E}^{\pi^t} \left[2\bont(s_h,a_h)+\nu_h^{t}(s_h,a_h)\mid s_{1}=s\right] \\
    & \le e \sum_{h=1}^H \mathbb{E}^{\pi^t} \left[2\bont(s_h,a_h)+\nu_h^{t}(s_h,a_h)\mid s_{1}=s\right].
\end{align*}
Finally with probability $\geq 1-\delta$ under $\mathbb{P}(\cdot \mid \mathcal{E}^{\mathrm{Off}}_\delta\cap \mathcal{E}^{\mathrm{On}}_\delta)$ we have that :  
\begin{align*}
V^\star_{1}(s)-V^{\pi^t}_{1}(s) < \widehat{V}_{1}^t(s)-V^{\pi^t}_{1}(s) = \overline{\Delta}_{1}(s) \leq e \sum_{h=1}^H \mathbb{E}^{\pi^t} \left[2\bont(s_h,a_h)+\nu_h^{t}(s_h,a_h)\mid s_{1}=s\right].
\end{align*}
meaning $\mathbb{P}(\mathcal{E}^{\mathrm{Conc}}_\delta\mid \mathcal{E}^{\mathrm{Off}}_\delta\cap \mathcal{E}^{\mathrm{On}}_\delta) \geq 1-\delta$, finally 
\begin{align*}
    \mathbb{P}\left (\mathcal{E}^{\mathrm{Off}}_\delta\cap \mathcal{E}^{\mathrm{On}}_\delta \cap \mathcal{E}^{\mathrm{Conc}}_\delta \right ) & = \mathbb{P}\left (\mathcal{E}^{\mathrm{Conc}}_\delta \mid \mathcal{E}^{\mathrm{Off}}_\delta\cap \mathcal{E}^{\mathrm{On}}_\delta  \right ) \mathbb{P}\left (\mathcal{E}^{\mathrm{Off}}_\delta\cap \mathcal{E}^{\mathrm{On}}_\delta  \right ) \\
    & \geq (1-\delta)(1-2\delta) \geq 1 - 3\delta
\end{align*}

\end{proof}

\begin{proposition}[High-Probability Regret Decomposition]\label{prop:regret_decomp_whp}
Note the event:
\[
\mathcal{E}^{\mathrm{Mar}}_\delta \coloneq \left \lbrace \eqref{eq:regret_decomp_whp} ~ \mathrm{holds} \right \rbrace 
\]
where \eqref{eq:regret_decomp_whp} is
\begin{equation}\label{eq:regret_decomp_whp}
\mathrm{Regret}(T) 
\le
e \sum_{t=1}^{T}\sum_{h=1}^{H}\left(2b_h^{t}(S_h^t,A_h^t)+ \nu_h^{t}(S_h^t,A_h^t)\right)
+
eHG\sqrt{2T\ln \dfrac{2}{\delta}},
\end{equation}
and $G$ is any constant such that, for all $(t,h,s,a)$ and almost surely,
\begin{equation}\label{eq:def_G}
0 \le 2b_h^{t}(s,a)+ \nu_h^{t}(s,a)
\le
G
\end{equation}
In particular on $\mathcal{E}^{\mathrm{Off}}_\delta\cap\mathcal{E}^{\mathrm{On}}_\delta$, one can use
\begin{equation}
\label{eq:G_def}
G \coloneqq D^{\max}\left[c_1\sqrt{L}+3|\mathcal{S}|HL_3\right]
+R^{\max}\left[c_1\sqrt{L}+2c_2L\right].
\end{equation}
We have 
\[
\mathbb{P} \left ( \mathcal{E}^{\mathrm{Off}}_\delta\cap \mathcal{E}^{\mathrm{On}}_\delta \cap \mathcal{E}^{\mathrm{Conc}}_\delta \cap \mathcal{E}^{\mathrm{Mar}}_\delta  \right ) > 1 - 4\delta.
\]
\end{proposition}

\begin{proof}
On the event \(\mathcal{E}^{\mathrm{Off}}_\delta\cap \mathcal{E}^{\mathrm{On}}_\delta \cap \mathcal{E}^{\mathrm{Conc}}_\delta \), under which Proposition~\ref{prop:regret_decomp} holds. For episode \(t\) define
\[
X_t:=\sum_{h=1}^H \left(2b_h^{t}(S_h^t,A_h^t)+ \nu_h^{t}(S_h^t,A_h^t)\right),
\]
and let \(\{\mathcal{F}_t\}_{t\ge 0}\) be the episode filtration from Definition~\ref{def:episode_filtration}. For every deterministic initial state \(s\), Proposition~\ref{prop:regret_decomp} gives
\[
V_1^\star(s)-V_1^{\pi^t}(s)\le e\mathbb{E}^{\pi^t}\Big[\sum_{h=1}^H \left(2b_h^{t}(s_h,a_h)+ \nu_h^{t}(s_h,a_h)\right)\Big|s_1=s\Big].
\]
Set \(s=S_1^t\), by Definition~\ref{def:episode_filtration}, for each \(h\) the objects \((2b_h^{t}+ \nu_h^{t})\) and the data dependent policy \(\pi^t\) are \(\mathcal{F}_{t-1}\)-measurable. Hence, almost surely,
\begin{equation}
\label{eq:step-given-G-formal}
V_1^\star(S_1^t)-V_1^{\pi^t}(S_1^t)
\le
e\mathbb{E}^{\pi^t}\Big[\sum_{h=1}^{H} \left(2b_h^{t}(s_h,a_h)+ \nu_h^{t}(s_h,a_h)\right)\Big|s_1=S_1^t\Big]
=
e\mathbb{E}\big[X_t\mid\sigma\big(\mathcal{F}_{t-1}\cup\sigma(S_1^t)\big)\big].
\end{equation}

Then, we obtain
\begin{equation}
\mathrm{Regret}(T)=\sum_{t=1}^T\big(V_1^\star(S_1^t)-V_1^{\pi^t}(S_1^t)\big)\le e\sum_{t=1}^T \mathbb{E}\big[X_t\mid \sigma\big(\mathcal{F}_{t-1}\cup\sigma(S_1^t)\big)\big].
\end{equation}

Define the martingale
\[
M_T:=\sum_{t=1}^T\Big(X_t-\mathbb{E}\big[X_t\mid \sigma\big(\mathcal{F}_{t-1}\cup\sigma(S_1^t)\big)\big]\Big),\quad T\ge 0,
\]
with differences \(D_t:=X_t-\mathbb{E}\big[X_t\mid \sigma\big(\mathcal{F}_{t-1}\cup\sigma(S_1^t)\big)\big]\). Since \(0\le 2b_h^{t}(s,a)+ \nu_h^{t}(s,a)\le G\) almost surely, we have \(0\le X_t\le HG\) and therefore \(|D_t|\le HG\). By Azuma-Hoeffding (Theorem~\ref{thm:AH}), with probability at least \(1-\delta\) under the law $\mathbb{P}\left (\cdot \mid \mathcal{E}^{\mathrm{Off}}_\delta\cap \mathcal{E}^{\mathrm{On}}_\delta \cap \mathcal{E}^{\mathrm{Conc}}_\delta\right )$ we have 
\[
|M_T|\le HG\sqrt{2T\ln\frac{2}{\delta}}.
\]
and using \(\sum_{t=1}^T \mathbb{E}\big[X_t\mid \sigma\big(\mathcal{F}_{t-1}\cup\sigma(S_1^t)\big)\big]=\sum_{t=1}^T X_t - M_T\), we obtain
\begin{align*}
\mathrm{Regret}(T)\le e\sum_{t=1}^T X_t + e|M_T|
& = e\sum_{t=1}^{T}\sum_{h=1}^{H} \left(2b_h^{t}(S_h^t,A_h^t)+ \nu_h^{t}(S_h^t,A_h^t)\right) + e|M_T| \\
& \leq e\sum_{t=1}^{T}\sum_{h=1}^{H} \left(2b_h^{t}(S_h^t,A_h^t)+ \nu_h^{t}(S_h^t,A_h^t)\right) + eHG\sqrt{2T\ln\frac{2}{\delta}}.
\end{align*}
Meaning 
\[\mathbb{P}\left (\mathcal{E}^{\mathrm{Mar}}_\delta \mid \mathcal{E}^{\mathrm{Off}}_\delta\cap \mathcal{E}^{\mathrm{On}}_\delta \cap \mathcal{E}^{\mathrm{Conc}}_\delta\right ) \geq 1-\delta\]

Finally 
\begin{align*}
    \mathbb{P}\left (\mathcal{E}^{\mathrm{Off}}_\delta\cap \mathcal{E}^{\mathrm{On}}_\delta \cap \mathcal{E}^{\mathrm{Conc}}_\delta \cap \mathcal{E}^{\mathrm{Mar}}_\delta \right ) & = \mathbb{P}\left (\mathcal{E}^{\mathrm{Mar}}_\delta \mid \mathcal{E}^{\mathrm{Off}}_\delta\cap \mathcal{E}^{\mathrm{On}}_\delta \cap \mathcal{E}^{\mathrm{Conc}}_\delta  \right ) \mathbb{P}\left (\mathcal{E}^{\mathrm{Off}}_\delta\cap \mathcal{E}^{\mathrm{On}}_\delta \cap \mathcal{E}^{\mathrm{Conc}}_\delta \right ) \\
    & \geq (1-\delta)(1-3\delta) \geq 1 - 4\delta
\end{align*}

\end{proof}
\subsection{Proof of Theorem \ref{thm:v_shaping_regret}}
\begin{theorem}[V-shaping Regret]
    The regret of Algorithm \ref{alg:v_shaping} can be bounded as follows:
    \begin{align*}
    \mathrm{Regret}(T) & \leq \tilde O\Big(
(R^{\max}+D^{\max})\sqrt{T H |\mathcal{S}\setminus \PPS| |\mathcal{A}|}
+\sqrt{T}(R^{\max}+|\mathcal{S}| H D^{\max}) \\
& + R^{\max} |\mathcal{S}\setminus \PPS| |\mathcal{A}| + D^{\max} |\mathcal{S}\setminus \PPS|^2 |\mathcal{A}| H\\
& + \min\{\Theta_{1}(\Delta),\Theta_{2}(\Delta)\}
\Big)
\end{align*}
where
\begin{align*}
\Theta_{1}(\Delta)
&=\dfrac{R^{\max}+D^{\max}}{\Delta}\sqrt{|\BPS||\mathcal{S}| |\mathcal{A}|H} + D^{\max}H|\mathcal{A}||\mathcal{S}^2| + R^{\max}|\mathcal{S}||\mathcal{A|}\\
\Theta_{2}(\Delta)
&=\dfrac{(1 + R^{\max})^2}{\Delta^2}|\BPS|(R^{\max}+|\mathcal{S}| H D^{\max})H
\end{align*}
\end{theorem}
\begin{proof}
Let $G>0$ satisfying \eqref{eq:G_def}.
Recall that $
\mathcal{U}=(\mathcal{S}\times\mathcal{A}) \setminus (\PPS\times\mathcal{A}
\cup \PS)$, $D^{\max} = \max_h D_{h+1}^{\max}$, and $R^{\max} = \max_h R_{h+1}$. Starting from decomposition \eqref{eq:regret_decomp_whp}, we get
\begin{align*}
\mathrm{Regret}(T)
&\le
2e\sum_{t=1}^{T}\sum_{h=1}^{H}
\mathbf{1}\{(S_h^t,A_h^t)\in\mathcal{U}\}b_h^t(S_h^t,A_h^t)
\quad\text{bounded using \eqref{eq:bonus_bound}, \eqref{eq:U_sum_sqrt_N_inv}, \eqref{eq:U_sum_N_inv}}\\
&+
2e\sum_{t=1}^{T}\sum_{h=1}^{H}
\mathbf{1}\{(S_h^t,A_h^t)\notin\mathcal{U}\}b_h^t(S_h^t,A_h^t)
\quad\text{bounded using \eqref{eq:bonus_bound}, \eqref{eq:U_sum_sqrt_N_inv}, \eqref{eq:U_sum_N_inv}, \eqref{eq:total_visitation_via_boundary}}\\
&+
e\sum_{t=1}^{T}\sum_{h=1}^{H}
\mathbf{1}\{(S_h^t,A_h^t)\in\mathcal{U}\}
\nu_h^{t}(S_h^t,A_h^t)
\quad\text{bounded using \eqref{eq:U_sum_N_inv}, \eqref{eq:U_neighbour_condition}}\\
&+
e\sum_{t=1}^{T}\sum_{h=1}^{H}
\mathbf{1}\{(S_h^t,A_h^t)\notin\mathcal{U}\}
\nu_h^{t}(S_h^t,A_h^t)
\quad\text{bounded using \eqref{eq:U_sum_N_inv}}\\
&+
HG\sqrt{2T\log \dfrac{2}{\delta}}
\\
& \leq 2e\bigg[c_1\Big(\dfrac12R^{\max}+\dfrac12 D^{\max}\Big)\sqrt{L}\cdot 2\sqrt{TH|\mathcal{U}|}
+ c_2R^{\max}L\cdot 2|\mathcal{U}|\ln(1+T)\bigg]\\
&+ 2e\Big[c_1\Big(\dfrac12R^{\max}+\dfrac12 D^{\max}\Big)\sqrt{L} \cdot 2\sqrt{|\mathcal{U}^{C}|H|\BPS|n(\Delta)}
+ c_2R^{\max}L\cdot 2 |\mathcal{U}^{C}|\ln{(1+T)}\Big]\\
&+ e\cdot D^{\max}\cdot
|\mathcal{S}\setminus \PPS| 3HL_3\cdot 2 |\mathcal{U}|\ln{(1+T)}\\
&+ e\cdot D^{\max}\cdot
|\mathcal{S}| 3HL_3\cdot 2 |\mathcal{U}^{C}|\ln{(1+T)}\\
&+ eHG\sqrt{2T\log \dfrac{2}{\delta}}.
\end{align*}
Another route is decomposing with respect to trajectories that intersect 
$\mathcal{U}$:
\begin{align*}
\mathrm{Regret}(T) 
    &\le 
    e \sum_{t=1}^{T} \mathbf{1}\left\{\tau_{t}\cap \mathcal{U}\neq\emptyset\right\}\sum_{h=1}^{H} \left(2b_h^{t}(S_h^t,A_h^t)+ \nu_h^{t}(S_h^t,A_h^t)\right)\\
    &+ e \sum_{t=1}^{T} \mathbf{1}\left\{\tau_{t}\cap \mathcal{U}=\emptyset\right\}\sum_{h=1}^{H} \left(2b_h^{t}(S_h^t,A_h^t)+ \nu_h^{t}(S_h^t,A_h^t)\right) \quad\text{bounded using \eqref{eq:G_def}}\\
    &+HG\sqrt{2T\log \frac{2}{\delta}}\\
    &\le e \sum_{t=1}^{T}\sum_{h=1}^{H}
\mathbf{1}\{(S_h^t,A_h^t)\in\mathcal{U}\} \left(2b_h^{t}(S_h^t,A_h^t)+ \nu_h^{t}(S_h^t,A_h^t)\right) \quad\text{bounded using above result}\\
\\
    &+ e \sum_{t=1}^{T} \mathbf{1}\left\{\tau_{t}\cap\Gamma\neq\emptyset\right\}HG\quad\text{bounded using \eqref{eq:trajectory_intersection_bound}}\\
\\
    &+HG\sqrt{2T\log \frac{2}{\delta}}\\
    &\le 2e\Big[c_1\Big(\dfrac12R^{\max}+\dfrac12 D^{\max}\Big)\sqrt{L}\cdot 2\sqrt{TH|\mathcal{U}|}
+ c_2R^{\max}L\cdot 2|\mathcal{U}|\ln(1+T)\Big]\\
    &+ e\cdot D^{\max}\cdot
|\mathcal{S}\setminus \PPS| 3HL_3\cdot 2 |\mathcal{U}|\ln{(1+T)}\\
    &+ eHG|\BPS| n(\Delta)\\
    &+eHG\sqrt{2T\log \frac{2}{\delta}}\\
\end{align*}

Finally, using the bounds $|\mathcal{U}^{C}| \leq |\mathcal{S}||\mathcal{A}|$,
$|\mathcal{U}|\le|\mathcal{S}\setminus \PPS||\mathcal{A}|$,
and bounding $G$ and $n(\Delta)$ respectively with (\ref{eq:G_def}) and \eqref{eq:total_visitation_pseudosub}, we get 

\begin{align*}
\sum_{t=1}^{T}\bigl(V^{\star}(s_{0})-V^{\pi_{t}}(s_{0})\bigr)
& =\tilde O\Big(
(R^{\max}+D^{\max})\sqrt{T H |\mathcal{S}\setminus \PPS| |\mathcal{A}|}
+\sqrt{T}(R^{\max}+|\mathcal{S}| H D^{\max}) \\
& + R^{\max} |\mathcal{S}\setminus \PPS| |\mathcal{A}| + D^{\max} |\mathcal{S}\setminus \PPS|^2 |\mathcal{A}| H\\
& + \min\{\Theta_{1}(\Delta),\Theta_{2}(\Delta)\}
\Big)
\end{align*}
\[
\Theta_{1}(\Delta)=\dfrac{R^{\max}+D^{\max}}{\Delta}\sqrt{|\BPS||\mathcal{S}| |\mathcal{A}|H} + D^{\max}H|\mathcal{A}||\mathcal{S}^2| + R^{\max}|\mathcal{S}||\mathcal{A|}
\]
\[
\Theta_{2}(\Delta)=\dfrac{(1 + R^{\max})^2}{\Delta^2}|\BPS|(R^{\max}+|\mathcal{S}| H D^{\max})H
\]
\end{proof}

\section{Q-shaping: proof of Theorem \ref{thm:q_shaping}}

\begin{definition}[Action Value Function \sandwich{}]
The algorithm is given access to known functions $\lowv_{h}, \highv_{h}: \mathcal{S}_h \times \mathcal{A} \to \mathbb{R}$ for each step $h$, which provide a \sandwich{} condition on the optimal value function:
\[
    \forall (s,a,h), \quad \lowq_h(s,a) \le Q_h^\star(s,a) \le \highq_h(s,a).
\]
taking the $\argmax$ over actions includes the previous Definition~\ref{assumption:sandwich}:
\[
    \forall (s,a,h), \quad \max_{a} \lowq_h(s,a) =: \lowv_{h}(s) \le V_h^\star(s) \le \highv_{h}(s) := \max_{a}\highq_h(s,a).
\]
\end{definition}

\begin{theorem}[Q-shaping Regret]
    The regret of Algorithm \ref{alg:q_shaping} can be bounded as follows:
    \[
    \mathrm{Regret}(T)\le R^{\max}\Gamma_R(T)+D^{\max}\Gamma_D(T),
    \]
with
    \[ \Gamma_R(T)=2ec_1\sqrt{L}\sqrt{TH|\mathrm{PairEff}|}+4ec_2L|\mathrm{PairEff}|\ln(1+T)+\left(c_1\sqrt{L}+2c_2L\right)\sqrt{2T\log\dfrac{2}{\delta}},
    \]
    \[
    \Gamma_D(T)=2ec_1\sqrt{L}\sqrt{TH|\mathrm{PairEff}|}+6e|\mathcal{S}|H|\mathrm{PairEff}|L_3\ln(1+T)+\left(c_1\sqrt{L}+3|\mathcal{S}|HL_3\right)\sqrt{2T\log\dfrac{2}{\delta}},
    \]
    where 
    \[ \mathrm{PairEff} \coloneq \left \lbrace (s,a) \in \mathcal{S} \times \mathcal{A} ~ | ~ \highq_{h(s)}(s,a) \geq V^{\star}_{h(s)}(s) \right \rbrace \]
\end{theorem}
\begin{proof}
    For all stages $h$ and states $s$: 
\begin{align}
    \mathrm{Support}(\pi^t_h(\cdot|s)) & \subseteq \left \lbrace (s,a) \in \mathcal{S} \times \mathcal{A} ~ | ~ \highq_{h(s)}(s,a) \geq V^{\star}_{h(s)}(s) \right \rbrace = \mathrm{PairEff}  \label{eq:support_argument}
\end{align}
This is because when clipping the $Q$-values this way, the $Q$-estimates still maintain optimism meaning: 
\[
\forall (s,a,h,t) \in \mathcal{S}  \times \mathcal{A} \times [H] \times [T], \quad Q_h^{\star}(s,a) \leq \widehat{Q}_h^t(s,a) \quad\text{and}\quad V_h^{\star}(s) \leq \widehat{V}_h^t(s) 
\]
If
\[ 
    (s,a) \in \left \lbrace (s,a) \in \mathcal{S} \times \mathcal{A} ~ | ~ \highq_{h(s)}(s,a) < V^{\star}_{h(s)}(s) \right \rbrace = \left ( \mathcal{S} \times \mathcal{A} \right ) \setminus \mathrm{PairEff},
\]
then 
\begin{align*}
    \widehat{Q}_h^t(s,a)  & \leq \highq_h(s,a) \\
    & < V^{\star}_{h(s)}(s) = Q^{\star}_{h(s)}(s, \pi^{\star}(s)) \\
    & < \widehat{Q}_h^t(s,\pi^{\star}(s)).
\end{align*}
where the first inequality is by the clipping mecanism in Algorithm~\ref{alg:q_shaping}, the second is by the assumption on $(s,a)$ and the third by optimism. Thus in $\left (s,a,h(s)\right )$, action $a$ will never be the selected action as the policy is greedy.

We use the same bonus as in \eqref{eq:bonus_bound}, the results from \ref{lem:on_bonus_validity}, to prove that optimism still holds, a slight modification must be made. By backward induction on $h$ with base case $\widehat{V}_{H+1}^t =  0$ and $V_{H+1}^\star =  0$, so $V_{H+1}^\star\le \widehat{V}_{H+1}^t$ holds. Fix $h\in\{1,\dots,H\}$ and assume $\widehat{V}_{h+1}^t(s')\ge V_{h+1}^\star(s')$ for all $s'$. For any $(s,a)$,
\[
 r_h(s,a)+\langle \widehat{P}^{t}, \widehat{V}_{h+1}^t\rangle + b_h^t(s,a)
\ge r_h(s,a)+\langle \widehat{P}^{t}, V_{h+1}^\star\rangle + b_h^t(s,a),
\]
where the inequality uses our induction hypothesis. Using the property:
\[
\bont(s,a) \ge \bigl\langle P_h^\star-\widehat{P}^{t},\,V_{h+1}^\star\bigr\rangle(s,a),
\]
hence
\[
\langle \widehat{P}^{t}, V_{h+1}^\star\rangle + b_h^t(s,a)
\ge \langle \widehat{P}^{t}, V_{h+1}^\star\rangle + \langle P_h^\star-\widehat{P}^{t}, V_{h+1}^\star\rangle
= \langle P_h^\star, V_{h+1}^\star\rangle.
\]
Combining these gives us 
\[
 r_h(s,a)+\langle \widehat{P}^{t}, \widehat{V}_{h+1}^t\rangle + b_h^t(s,a) \ge r_h(s,a)+\langle P_h^\star, V_{h+1}^\star\rangle = Q_h^\star(s,a).
\]
By our assumption, we also have $Q_h^\star(s,a)\le \highq_h(s,a)$, therefore
\[
Q_h^t(s,a)=\min\{ r_h(s,a)+\langle \widehat{P}^{t}, \widehat{V}_{h+1}^t\rangle + b_h^t(s,a),\highq_h(s,a)\}\ge Q_h^\star(s,a).
\]
Taking the maximum over $a$ yields
\[
\widehat{V}_h^t(s)=\max_a Q_h^t(s,a)\ge \max_a Q_h^\star(s,a)=V_h^\star(s).
\]
The fact that optimism holds gives us the applicability of Proposition~\ref{prop:regret_decomp_whp} on its respective event \(\mathcal{E}^{\mathrm{Off}}_\delta\cap \mathcal{E}^{\mathrm{On}}_\delta \cap \mathcal{E}^{\mathrm{Conc}}_\delta \cap \mathcal{E}^{\mathrm{Mar}}_\delta \):
\begin{align*}
    \mathrm{Regret}(T) & \leq e\sum_{t=1}^{T}\sum_{h=1}^{H} \left(2b_h^{t}(S_h^t,A_h^t)+ \nu_h^{t}(S_h^t,A_h^t)\right)\mathbf{1}\{(S_h^t,A_h^t)\in\mathrm{PairEff}\} +HG\sqrt{2T\log \frac{2}{\delta}}\\
    & \le 2e\sum_{t=1}^{T}\sum_{h=1}^{H} b_h^{t}(S_h^t,A_h^t)\mathbf{1}\{(S_h^t,A_h^t)\in\mathrm{PairEff}\}\quad\text{bounded using \eqref{eq:bonus_bound}, \eqref{eq:U_sum_sqrt_N_inv}, \eqref{eq:U_sum_N_inv}, \eqref{eq:support_argument}}\\
    &+e\sum_{t=1}^{T}\sum_{h=1}^{H}D_{h+1}^{\max} d^{S_h^t,A_h^t}
    \frac{3 H L_3}
    {N_{h(S_h^t)}^t(S_h^t,A_h^t)}\mathbf{1}\{(S_h^t,A_h^t)\in\mathrm{PairEff}\}\quad\text{bounded using \eqref{eq:U_sum_sqrt_N_inv}, \eqref{eq:U_sum_N_inv}, \eqref{eq:support_argument}}\\
    &+eHG\sqrt{2T\log \frac{2}{\delta}}\\
    & \leq 2e\Big[c_1\Big(\dfrac12R^{\max}+\dfrac12 D^{\max}\Big)\sqrt{L} \cdot 2\sqrt{TH|\mathrm{PairEff}|}+ c_2R^{\max}L\cdot 2 |\mathrm{PairEff}|\ln(1+T)\Big)\Big] \\
&+ eD^{\max}\cdot3|\mathcal{S}|HL_{3}\cdot2|\mathrm{PairEff}|\ln(1+T)\\
&+eHG\sqrt{2T\log \frac{2}{\delta}},
\end{align*}
where we used that $2b_h^{t}(S_h^t,A_h^t)+ \nu_h^{t}(S_h^t,A_h^t)$ is almost surely bounded by $G$. Replacing $G$ using \ref{eq:G_def} we get: 
\[
\mathrm{Regret}(T)\le R^{\max}\Gamma_R(T)+D^{\max}\Gamma_D(T),
\]
with
\[
\Gamma_R(T)=2ec_1\sqrt{L}\sqrt{TH|\mathrm{PairEff}|}+4ec_2L|\mathrm{PairEff}|\ln(1+T)+e\left(c_1\sqrt{L}+2c_2L\right)\sqrt{2T\log\dfrac{2}{\delta}},
\]
\[
\Gamma_D(T)=2ec_1\sqrt{L}\sqrt{TH|\mathrm{PairEff}|}+6e|\mathcal{S}|H|\mathrm{PairEff}|L_3\ln(1+T)+e\left(c_1\sqrt{L}+3|\mathcal{S}|HL_3\right)\sqrt{2T\log\dfrac{2}{\delta}}.
\]
Then,
\[
\mathrm{Regret}(T)
\le
R^{\max}\tilde{\mathcal O}\Big(\sqrt{TH|\mathrm{PairEff}|}+|\mathrm{PairEff}|+\sqrt{T}\Big)
+
D^{\max}\tilde{\mathcal O}\Big(\sqrt{TH|\mathrm{PairEff}|}+|\mathcal{S}|H|\mathrm{PairEff}|+|\mathcal{S}|H\sqrt{T}\Big),
\]
where $\tilde{\mathcal O}$ hides polylog factors in $T,|\mathcal{S}|,|\mathcal{A}|,H$,$1/\delta$.
\end{proof}

\section{Proof of Theorem~\ref{thm:offline_q_shaping}}
\label{sec:final_results}

The following two theorems hold when 
\[
     K\ge \frac{8H}{d^{\mathsf{b}}_{\min}}\log\left(\frac{H|\mathcal{S}||\mathcal{A}|}{\delta}\right).
\]

\begin{theorem}[Offline V-Shaping]
    Replacing $D^{\max}$ in Theorem~\ref{thm:v_shaping_regret}'s regret with the bound from Proposition~\ref{prop:offline_width} we get that with probability at least $1-4\delta$, the regret after $T$ episodes of running Algorithm~\ref{alg:v_shaping} using the output of Algorithm~\ref{alg:offline} as an envelope satisfies:
\begin{align*}
\sum_{t=1}^{T}\bigl(V^{\star}(s_{0})-V^{\pi_{t}}(s_{0})\bigr)
&=\tilde O\Bigg(
R^{\max}\sqrt{T H |\mathcal{S}\setminus \PPS| |\mathcal{A}|}
+\sqrt{T}\,R^{\max}
+ R^{\max} |\mathcal{S}\setminus \PPS| |\mathcal{A}|
\\
&\qquad
+ \frac{1}{\sqrt{K d_{\min}^{\mathsf b}}}\Big[
H^{5/2}\sqrt{T H |\mathcal{S}\setminus \PPS| |\mathcal{A}|}
+ H^{7/2}\sqrt{T}\,|\mathcal{S}|
+ H^{7/2}|\mathcal{S}\setminus \PPS|^{2} |\mathcal{A}|
\Big]
\quad (\propto K^{-1/2})
\\
&\qquad
+ \frac{1}{K d_{\min}^{\mathsf b}}\Big[
H^{3}\sqrt{T H |\mathcal{S}\setminus \PPS| |\mathcal{A}|}
+ H^{4}\sqrt{T}\,|\mathcal{S}|
+ H^{4}|\mathcal{S}\setminus \PPS|^{2} |\mathcal{A}|
\Big]
\quad (\propto K^{-1})
\\
&\qquad
+ \min\{\Theta_{1}(\Delta),\Theta_{2}(\Delta)\}
\Bigg).
\end{align*}

\begin{align*}
\Theta_{1}(\Delta)
&=
\frac{R^{\max}}{\Delta}\sqrt{|\BPS||\mathcal S||\mathcal A|H}
+ R^{\max}|\mathcal S||\mathcal A|
\\
&\qquad
+ \frac{1}{\sqrt{K d_{\min}^{\mathsf b}}}\Big[
\frac{H^{5/2}}{\Delta}\sqrt{|\BPS||\mathcal S||\mathcal A|H}
+ H^{7/2}|\mathcal A||\mathcal S|^{2}
\Big]
\quad (\propto K^{-1/2})
\\
&\qquad
+ \frac{1}{K d_{\min}^{\mathsf b}}\Big[
\frac{H^{3}}{\Delta}\sqrt{|\BPS||\mathcal S||\mathcal A|H}
+ H^{4}|\mathcal A||\mathcal S|^{2}
\Big]
\quad (\propto K^{-1}).
\end{align*}

\begin{align*}
\Theta_{2}(\Delta)
&=
\frac{(1+R^{\max})^{2}}{\Delta^{2}}\,|\BPS|\,R^{\max} H
\\
&\qquad
+ \frac{1}{\sqrt{K d_{\min}^{\mathsf b}}}\Big[
\frac{(1+R^{\max})^{2}}{\Delta^{2}}\,|\BPS|\,|\mathcal S|\,H^{9/2}
\Big]
\quad (\propto K^{-1/2})
\\
&\qquad
+ \frac{1}{K d_{\min}^{\mathsf b}}\Big[
\frac{(1+R^{\max})^{2}}{\Delta^{2}}\,|\BPS|\,|\mathcal S|\,H^{5}
\Big]
\quad (\propto K^{-1}).
\end{align*}

\end{theorem}

\begin{theorem}[Offline Q-Shaping]
      Replacing $D^{\max}$ in Theorem~\ref{thm:offline_q_shaping}'s regret with the bound from Proposition~\ref{prop:offline_width} we get that with probability at least $1-4\delta$, the regret after $T$ episodes of running Algorithm~\ref{alg:q_shaping} using the output of Algorithm~\ref{alg:offline} as an envelope satisfies:
\begin{align*}
    \mathrm{Regret}(T) & \leq 
R^{\max}\tilde{\mathcal O}\left( \sqrt{TH|\mathrm{PairEff}|} + |\mathrm{PairEff}| + \sqrt{T} \right)\\
&+ \tilde{\mathcal O}\left[ \left( H^{5/2}\frac{1}{\sqrt{K d_{\min}^{\mathsf b}}} + H^{3}\frac{1}{K d_{\min}^{\mathsf b}} \right)\left( \sqrt{TH|\mathrm{PairEff}|} + |\mathcal{S}|H|\mathrm{PairEff}| + |\mathcal{S}|H\sqrt{T} \right) \right]
\end{align*}
\end{theorem}
\begin{proof}
     Replacing $D^{\max}$ in Theorem~\ref{thm:offline_q_shaping}'s regret with the bound from Proposition~\ref{prop:offline_width} gives us the result.
\end{proof}

\section{Technical Lemmas}

\begin{theorem}[Empirical Bernstein Inequality~\cite{maurer_empirical_2009}]
\label{thm:empirical_bernstein}
Let $Z,Z_{1},\dots,Z_{n}$ be independent and identically distributed random variables
taking values in $[0,1]$. Let $n\ge 2$ and fix $\delta\in(0,1)$. Let
$\mathbf{Z}=(Z_{1},\dots,Z_{n})$ and define the \textbf{unbiased} sample variance
\[
V_{n}(\mathbf{Z})
 = 
\frac{1}{n(n-1)}\sum_{1\le i<j\le n}(Z_{i}-Z_{j})^{2}
= 
\frac{1}{n-1}\sum_{i=1}^n\Big(Z_i -  \frac{1}{n}\sum_{i=1}^{n} Z_{i}\Big)^2
\]
Then, 
\[
\mathbb{P}\left( \mathbb{E}[Z]  -  \frac{1}{n}\sum_{i=1}^{n} Z_{i}
 \le 
\sqrt{\frac{2 V_{n}(\mathbf{Z}) \ln(2/\delta)}{n}}
 + 
\frac{7 \ln(2/\delta)}{3 (n-1)} \right ) \ge 1 - \delta
\]
Noting the \textbf{biased} sample variance
\[\widehat{\mathrm{Var}}(\mathbf{Z})
 = 
\frac{1}{n}\sum_{i=1}^n\Big(Z_i -  \frac{1}{n}\sum_{i=1}^{n} Z_{i}\Big)^2
\]
if we apply to $Z_i$ and $-Z_i$ and use \( \dfrac{1}{n-1} \leq \dfrac{2}{n} \) for \( n \geq 2\), we get the following form:
\[
\mathbb{P}\left( \left | \mathbb{E}[Z]  -  \frac{1}{n}\sum_{i=1}^{n} Z_{i} \right |
 \le 
2\sqrt{\frac{ \widehat{\mathrm{Var}}(\mathbf{Z}) \ln(4/\delta)}{n}}
 + 
\frac{14}{3}\frac{\ln(4/\delta)}{n} \right ) \ge 1 - \delta
\]
\end{theorem}

\begin{theorem}[Azuma-Hoeffding Inequality~\cite{hoeffding_probability_1963, azuma_weighted_1967}]
\label{thm:AH}
Let $\{M_n\}_{n\ge 0}$ be a martingale with differences $D_n := M_n - M_{n-1}$ satisfying
$|D_n|\le c_n$ almost surely.
Then, for all $n\ge 1$ and all $\varepsilon>0$,
\[
\mathbb{P}\left(|M_n|\ge \varepsilon\right)\le 2\exp\left(-\frac{\varepsilon^2}{2\sum_{i=1}^n c_i^2}\right).
\]
Equivalently, for all $n\ge 1$ and all $\delta\in(0,1)$,
\[
\mathbb{P}\left(|M_n|\le \sqrt{2\left(\sum_{i=1}^n c_i^2\right)\ln\frac{2}{\delta}}\right)\ge 1-\delta.
\]
\end{theorem}

\begin{theorem}[Multiplicative Chernoff for Bernoulli Sums \cite{chernoff_measure_1952}]
\label{thm:chernoff}
Suppose $X_1, \dots, X_n$ are independent random variables taking values in $\{0, 1\}$. Let $X = \sum_{i=1}^n X_i$ and $\mu = \mathbb{E}[X]$. For any $0 < \delta < 1$,
\begin{align*}
\Pr(X \leq (1-\delta)\mu) & \leq \left(\frac{e^{-\delta}}{(1-\delta)^{1-\delta}}\right)^\mu \\
    & \leq e^{-\delta^2 \mu/2}.
\end{align*}
\end{theorem}

The following results are in the context of an online setting using definitions from Subsection~\ref{sec:counts}.
\begin{lemma}[\cite{gupta_unpacking_2022} Corollary B.7]
    \label{lemma:borrowed}
With probability at least $1-\delta$ for all $t \in \mathbb{N}, s \in \mathcal{S}, a \in \mathcal{A}$ and all $f : \mathcal{S} \to [0,B]$, we have
\[
\left| 
    \left \langle \widehat{P}_h^{t}(\cdot|s,a) - P^\star(\cdot|s,a) , f \right \rangle
\right|
\le 
\frac{\mathbb{E}_{s' \sim P^\star(\cdot|s,a)}\left[f(s')\right]}{H}
+ 
B d^{s,a}
    \frac{3   H L_3}
    {N_{h(s)}^t(s,a)}
\]
Where 
\[d^{s,a} = d_{h(s)}^{s,a} \coloneqq |\mathrm{support}(P^\star(\cdot|s,a))|\] 
and $h(s)$ corresponds to the horizon index of the state partitions that contain state $s$.
\end{lemma}

\begin{lemma}[\cite{gupta_unpacking_2022} Lemma C.5]
For any set \( \Gamma \subseteq \mathcal{S}\times\mathcal{A} \), we have
\begin{align}
    \sum_{t,h}\frac{\mathbf{1}\{(S_h^t,A_h^t)\in\Gamma \}}{N_h^t(S_h^t,A_h^t)}
    &\le 2 \log \left ( 1 + \sum_{(s,a)\in\Gamma }N^{T}_{h(s)}(s,a)\right )\le  2|\Gamma |\log(1+T).
    \label{eq:U_sum_N_inv}
\end{align}
\end{lemma}

\begin{lemma}[\cite{agarwal2019reinforcement} Lemma 7.5]
For any set \( \Gamma \subseteq \mathcal{S}\times\mathcal{A} \), we have
\begin{align}
    \sum_{t,h}\frac{\mathbf{1}\{(S_h^t,A_h^t)\in\Gamma \}}{\sqrt{N_h^t(S_h^t,A_h^t)}}
    &\le 2 \sqrt{|\Gamma | \sum_{(s,a)\in\Gamma }N^{T}_{h(s)}(s,a)}\le 2\sqrt{|\Gamma |TH}.
    \label{eq:U_sum_sqrt_N_inv}
\end{align}
\end{lemma}

\begin{proposition}[\cite{gupta_unpacking_2022} Lemma B.4]
For a given \( \Delta \), define 
\begin{equation}
\mathcal{U}
=(\mathcal{S}\times\mathcal{A}) \setminus (\PPS\times\mathcal{A}
\cup \PS)
\end{equation}
we get 
\begin{align}
    \forall (s,a)\in\mathcal{U}, \quad \mathrm{Neighbour}(s,a)
    &\subseteq \mathcal{S}\setminus \PPS
    \label{eq:U_neighbour_condition}
\end{align}
as well as
\begin{equation}\label{eq:trajectory_intersection_bound}
\sum_{t=1}^{T}
\mathbf{1}\{\tau_{t}\cap\mathcal{U}^{C}\neq\emptyset\}\le
|\BPS| n(\Delta)\quad \text{for } \Delta>0
\end{equation}
and
\begin{equation}\label{eq:total_visitation_via_boundary}
\sum_{(s,a)\in \mathcal{U}^{C}}N^{T}_{h(s)}(s,a)
\le H|\BPS| n(\Delta), 
\end{equation}
where we define 
\[n(\Delta)=\max_{h} n_h(\Delta),\]
and 
\[
\mathrm{Neighbour}(s,a) = \left \lbrace s' \in \mathcal{S} \vert ~ P^\star_h(s' \vert s,a) > 0 \right \rbrace. 
\]
\end{proposition}

\end{document}